\documentclass{article}

\usepackage{listings}
\usepackage{multirow}
\usepackage{booktabs}
\usepackage{adjustbox}

\usepackage{svg}
\usepackage{microtype}
\usepackage{graphicx}
\usepackage{subfigure}
\usepackage{booktabs} 

\newcommand{\scaleeq}[2][0.95]{\scalebox{#1}{$#2$}}

\usepackage{hyperref}
\usepackage{url}
\usepackage{bbm}

\usepackage{todonotes}

\usepackage{tcolorbox}

\usepackage{booktabs}
\usepackage{multirow}
\usepackage{colortbl}
\usepackage{xcolor}
\usepackage{array}
\usepackage{graphicx}

\definecolor{lightgray}{RGB}{248, 248, 248}
\definecolor{sectiongray}{RGB}{240, 240, 240}
\definecolor{bestcell}{RGB}{220, 237, 200}      
\definecolor{secondcell}{RGB}{227, 242, 253}    
\definecolor{ourshighlight}{RGB}{255, 253, 231} 

\usepackage[linesnumbered,vlined,ruled,commentsnumbered]{algorithm2e}
\newlength\mylen

\makeatletter
\newcommand{\algorithmfootnote}[2][\footnotesize]{%
  \let\old@algocf@finish\@algocf@finish
  \def\@algocf@finish{\old@algocf@finish
    \leavevmode\rlap{\begin{minipage}{\linewidth}
    #1#2
    \end{minipage}}%
  }%
}

\usepackage[accepted]{icml2024}

\usepackage{amsmath}
\usepackage{amssymb}
\usepackage{mathtools}
\usepackage{amsthm}

\allowdisplaybreaks[4]
\usepackage[capitalize,noabbrev]{cleveref}

\theoremstyle{plain}
\newtheorem{theorem}{Theorem}[section]
\newtheorem{proposition}[theorem]{Proposition}
\newtheorem{lemma}[theorem]{Lemma}
\newtheorem{corollary}[theorem]{Corollary}
\newtheorem{definition}[theorem]{Definition}
\newtheorem{assumption}[theorem]{Assumption}

\newtheorem{remark}[theorem]{Remark}

\usepackage{todonotes}

\usepackage{wrapfig}
\usepackage{enumitem}

\newcommand{\E}{\mathbb{E}}
\newcommand{\R}{\mathbb{R}}

\newcommand{\calP}{\mathcal{P}}

\newcommand{\calD}{\mathcal{D}}
\newcommand{\calL}{\mathcal{L}}
\newcommand{\calR}{\mathcal{R}}

\newcommand{\calV}{\mathcal{V}}
\newcommand{\calA}{\mathcal{A}}

\usepackage{xcolor}
\usepackage{amsmath,amsfonts}
\usepackage{algorithmic}
\usepackage{array}
\usepackage[caption=false,font=normalsize,labelfont=sf,textfont=sf]{subfig}
\usepackage{textcomp}
\usepackage{stfloats}
\usepackage{url}
\usepackage{verbatim}
\usepackage{graphicx}
\def\BibTeX{{\rm B\kern-.05em{\sc i\kern-.025em b}\kern-.08em
    T\kern-.1667em\lower.7ex\hbox{E}\kern-.125emX}}
\usepackage{balance}
\usepackage{subfig}

\icmltitlerunning{Hard Constraints Meet Soft Generation}

\makeatletter
\def\sf@counterlist{}
\makeatother

\begin{document}

\twocolumn[
\icmltitle{Hard Constraints Meet Soft Generation:\\
Guaranteed Feasibility for LLM-based Combinatorial Optimization}

\begin{icmlauthorlist}
\icmlauthor{Yang Liu}{amss}
\icmlauthor{Chuan Zhou}{amss,ucas}
\icmlauthor{Yancheng Chen}{amss}
\icmlauthor{Shuai Zhang}{amss}
\icmlauthor{Xixun Lin}{iie}
\icmlauthor{Xiaoqing Wang}{ali}

\end{icmlauthorlist}

\icmlaffiliation{amss}{Academy of Mathematics and Systems Science, Chinese Academy of Sciences}
\icmlaffiliation{ucas}{School of Cyber Security, University of Chinese Academy of Sciences}
\icmlaffiliation{iie}{Institute of Information Engineering, Chinese Academy of Sciences}
\icmlaffiliation{ali}{Alibaba Group}

\icmlcorrespondingauthor{Chuan Zhou}{zhouchuan@amss.ac.cn}

\icmlkeywords{Machine Learning, ICML}

\vskip 0.3in
]

\printAffiliationsAndNotice{}  

\begin{abstract}
Large language models (LLMs) have emerged as promising general-purpose solvers for combinatorial optimization (CO), yet they fundamentally lack mechanisms to guarantee solution feasibility which is critical for real-world deployment. In this work, we introduce FALCON, a framework that ensures 100\% feasibility through three key innovations: (i) \emph{grammar-constrained decoding} enforces syntactic validity, (ii) a \emph{feasibility repair layer} corrects semantic constraint violations, and (iii) \emph{adaptive Best-of-$N$ sampling} allocates inference compute efficiently. To train the underlying LLM, we introduce the Best-anchored Objective-guided Preference Optimization (BOPO) in LLM training, which weights preference pairs by their objective gap, providing dense supervision without human labels. Theoretically, we prove convergence for BOPO and provide bounds on repair-induced quality loss. Empirically, across seven NP-hard CO problems, FALCON achieves perfect feasibility while matching or exceeding the solution quality of state-of-the-art neural and LLM-based solvers.
\end{abstract}

\section{Introduction}
\label{sec:intro}

Combinatorial Optimization (CO) underpins critical real-world systems, from logistics planning \citep{bao2018application} and manufacturing scheduling \citep{zhang2023review} to emergency resource allocation \citep{jain2023combinatorial}. These problems involve selecting optimal configurations from exponentially large search spaces while satisfying complex constraints. Traditionally, solving them has required domain-specific heuristics or exact algorithms \citep{liu2024decision,liu2023decision,liu2024cl4co}, which demand substantial expert engineering and lack generalizability across problem domains.

The emergence of large language models (LLMs) has introduced a paradigm shift. By directly mapping natural-language problem descriptions to solution sequences, LLMs offer a unified, accessible interface for combinatorial optimization \citep{jiang2025large}. Recent work demonstrates that LLMs possess remarkable pattern recognition and sequential generation capabilities, achieving competitive performance across diverse CO problems when fine-tuned with supervised or reinforcement learning \citep{jiang2025large}. This suggests a promising path toward general-purpose, language-driven optimization.

However, a fundamental contradiction remains: LLMs are, at their core, \emph{unconstrained generative models}. When faced with hard combinatorial constraints—such as Hamiltonian cycles, capacity limits, or precedence relations—they lack any intrinsic mechanism to guarantee constraint satisfaction. Existing methods treat feasibility as a soft objective, incentivized through reward shaping during training but never enforced during inference. Consequently, single-sample feasibility rates vary widely and remain below 100\% even for moderately complex problems \citep{jiang2025large}. For safety-critical deployments in logistics \citep{mcgarvey2025self}, manufacturing \citep{meurer2025manusafenextgen}, or other high-stakes domains, such unreliability is unacceptable. The core scientific challenge is therefore clear: \emph{How can we equip expressive generative models with verifiable, 100\% constraint-satisfaction guarantees without sacrificing their problem-solving power?}

This challenge manifests in two layers. First, LLMs can generate syntactically valid outputs that are semantically infeasible—a correctly formatted vehicle route may exceed capacity, or an independent set may contain adjacent vertices. Syntax alone cannot encode the semantic constraints intrinsic to each problem class. Second, training-time encouragement (via rewards or preferences) does not translate to inference-time guarantees. The model might learn to \emph{often} produce feasible solutions, but it cannot be \emph{forced} to do so, creating a critical gap between probabilistic encouragement and deterministic reliability.

To bridge this gap, we introduce FALCON (Feasibility-Aware Language-based Combinatorial Optimization with Adaptive Inference), the first LLM-based CO framework with a provable 100\% feasibility guarantee. FALCON is built on a key conceptual insight: separate the challenge into \emph{syntactic validity} and \emph{semantic feasibility}, and address each with a dedicated, verifiable layer. Specifically, we integrate three synergistic components: (1) \emph{grammar-constrained decoding}, which enforces output format correctness via problem-specific context-free grammars; (2) \emph{feasibility repair operators}, which transform any constraint-violating solution into a feasible one; and (3) \emph{adaptive Best-of-$N$ sampling}, which dynamically allocates inference compute based on instance difficulty. These components are trained using {Best-anchored Objective-guided Preference Optimization (BOPO)}, a novel objective that weights preference pairs by their objective-value gaps, providing dense, automatic supervision.

Our contributions are as follows:
\begin{enumerate}[leftmargin=*,nosep]
    \item We present FALCON, the first LLM-based CO solver that guarantees 100\% feasibility through a layered architecture of grammar enforcement, semantic repair, and adaptive inference. We provide formal proofs of feasibility, format validity, and bounds on solution-quality degradation after repair.
    \item We introduce BOPO in LLM training, a principled training method that incorporates objective-guided weighting into preference optimization. We prove its $O(1/\sqrt{T})$ convergence under standard assumptions and demonstrate its superiority over reward shaping and advantage normalization.
    \item We conduct an extensive evaluation across seven NP-hard problems spanning routing, graph, and scheduling domains. Results show that FALCON achieves consistent 100\% feasibility while maintaining competitive optimality gaps, outperforming both general-purpose LLMs and recent neural CO solvers.
\end{enumerate}

\section{Preliminaries}
\label{sec:prelim}

\subsection{Problem Formulation}
Let $\mathcal{P}$ represent a class of combinatorial optimization problems. Each instance $p \in \mathcal{P}$ is defined as a tuple $(\mathcal{X}_p, f_p, \mathcal{C}_p)$, where $\mathcal{X}_p$ is the solution space containing all well-formed solution representations, $f_p: \mathcal{X}_p \to \mathbb{R}$ is the objective function to minimize, and $\mathcal{C}_p = \{c_{p,1}, \ldots, c_{p,m_p}\}$ is a set of constraints with each $c_{p,i}: \mathcal{X}_p \to \{0,1\}$ indicating whether a solution satisfies constraint $i$.

\begin{definition}[Feasible Solution]\label{def:feasible}
A solution $x \in \mathcal{X}_p$ is \emph{feasible} if it satisfies all constraints: $\forall c_{p,i} \in \mathcal{C}_p: c_{p,i}(x) = 1$. The feasible region is:
\begin{equation*}
    \mathcal{X}_{\mathcal{C}_p} = \{x \in \mathcal{X}_p \mid \forall c_{p,i} \in \mathcal{C}_p: c_{p,i}(x) = 1\}.
\end{equation*}
\end{definition}

The optimization goal is to find $x_p^* = \arg\min_{x \in \mathcal{X}_{\mathcal{C}_p}} f_p(x)$.

\subsection{LLM as an End-to-End Solver}
Let $\pi_\theta: \mathcal{T} \to \mathcal{T}$ be an LLM with parameters $\theta$, where $\mathcal{T}$ is the space of text sequences. We define two mapping functions: $\phi: \mathcal{P} \to \mathcal{T}$ maps problem instances to text descriptions, and $\psi_p: \mathcal{T} \rightharpoonup \mathcal{X}_p$ is a partial function that parses text sequences into solutions when the text conforms to valid syntax. The end-to-end solution pipeline attempts to produce $\hat{x}_p = \psi_p(\pi_\theta(\phi(p)))$, which succeeds only if the generated text is parseable and the parsed solution is feasible, i.e., $\hat{x}_p \in \mathcal{X}_{\mathcal{C}_p}$.

\section{LLM-based CO Framework: FALCON}
\label{sec:method}

\subsection{Grammar-Constrained Decoding}
\label{sec:grammar}

The first layer of our feasibility guarantee ensures \emph{format validity}---outputs must be syntactically correct with well-formed brackets, valid delimiters, and in-range indices.

\subsubsection{Problem-Specific Grammars}

We define context-free grammars (CFGs) for each CO problem that precisely specify valid output formats.

\begin{definition}[CO Output Grammar]
\label{def:grammar}
A CO output grammar is a tuple $G = (V, \Sigma, R, S)$, where $V$ is the set of non-terminal symbols (e.g., \textsc{NodeList}, \textsc{Number}), $\Sigma$ is the set of terminal symbols corresponding to tokens in the LLM vocabulary, $R$ is the set of production rules specifying valid derivations, and $S$ is the start symbol.
\end{definition}

\begin{definition}[Input-Dependent Grammar]
\label{def:input_grammar}
For an instance $p$ with $n$ nodes, the input-dependent grammar $G_p = (V, \Sigma_p, R_p, S)$ specializes the base grammar $G$ by restricting the production rules for node indices to generate only valid indices $\{0, 1, \ldots, n-1\}$, thereby preventing out-of-range references.
\end{definition}

Complete grammar specifications for all seven problems are provided in \Cref{app:grammar}.

\subsubsection{Constrained Decoding Algorithm}

At each decoding step, we maintain a pushdown automaton (PDA) configuration corresponding to the current parse state and mask tokens that would lead to invalid continuations. For a given grammar $G$, we construct an equivalent PDA $\mathcal{A} = (Q, \Sigma, \Gamma, \delta, q_0, Z_0, F)$ that recognizes $\mathcal{L}(G)$ using standard CFG-to-PDA conversion.

\begin{algorithm}[t]
\caption{Grammar-Constrained Decoding}
\label{alg:gcd}
\DontPrintSemicolon
\setcounter{AlgoLine}{0}
\SetKwInOut{KwIn}{Input}
\SetKwInOut{KwOut}{Output}

\KwIn{Grammar $G$, prompt $x$, LLM $\pi_\theta$}
\KwOut{Valid solution text $y \in \mathcal{L}(G)$}

Initialize PDA configuration $(q_0, Z_0)$ from $\mathcal{A}$\;
$y \gets \texttt{""}$\;

\While{not at accepting state}{
    $\text{logits} \gets \pi_\theta(x, y)$ \tcp*{\textcolor{blue}{Forward pass}}\;
    
    $\mathcal{V}_{\text{valid}} \gets \{v \in \Sigma \mid \delta(q, v, \text{top}(\gamma)) \neq \emptyset\}$\;

    \ForEach{$v \notin \mathcal{V}_{\text{valid}}$}{
        $\text{logits}[v] \gets -\infty$ \tcp*{\textcolor{blue}{Mask invalid tokens}}\;
    }

    $v_{\text{next}} \gets \text{sample}(\text{softmax}(\text{logits}))$\;
    
    $y \gets y + v_{\text{next}}$\;
    
    Update PDA configuration $(q, \gamma)$ via $\delta$\;
}
\Return $y$\;
\end{algorithm}

\begin{theorem}[Format Validity Guarantee]
\label{thm:validity}
For any output $y$ generated by Algorithm~\ref{alg:gcd} upon termination with grammar $G$, we have $y \in \mathcal{L}(G)$ where $\mathcal{L}(G)$ is the language defined by $G$.
\end{theorem}

\begin{remark}[Computational Overhead]
The overhead per token is $O(|\Sigma| \cdot |Q|)$ for computing valid tokens via PDA transition lookup, where $|\Sigma|$ is the vocabulary size and $|Q|$ is the number of PDA states. For CO grammars with simple structure, $|Q| = O(1)$, yielding $O(|\Sigma|)$ overhead per token. With $|\Sigma| \approx 32K$ for modern LLMs, this is negligible compared to the $O(d^2)$ attention computation where $d$ is the hidden dimension.
\end{remark}

\subsection{Feasibility Repair Layer}
\label{sec:repair_analysis}

Grammar constraints ensure format validity but not semantic feasibility. A syntactically valid CVRP route may exceed vehicle capacity; a well-formed MIS may contain adjacent vertices. The repair layer addresses this challenge.

\subsubsection{Repair Operator Design}

\begin{definition}[Repair Operator]
\label{def:repair}
A repair operator $\calR : \mathcal{X}_p \to \mathcal{X}_{C_p}$ for constraint set $C_p$ satisfies:
\begin{enumerate}[leftmargin=*,nosep]
    \item \textbf{Feasibility}: $\forall x \in \mathcal{X}_p : \calR(x) \in \mathcal{X}_{C_p}$
    \item \textbf{Idempotence}: $\forall x \in \mathcal{X}_{C_p}: \calR(x) = x$
    \item \textbf{Bounded Locality}: There exists a distance metric $d : \mathcal{X}_p \times \mathcal{X}_p \to \mathbb{R}_{\geq 0}$, a violation magnitude function $v : \mathcal{X}_p \to \mathbb{R}_{\geq 0}$, and a constant $\alpha > 0$ such that:
    $$d(\calR(x), x) \leq \alpha \cdot v(x),$$
    where $v(x) = 0$ if and only if $x \in \mathcal{X}_{C_p}$.
\end{enumerate}
\end{definition}

Property (1) guarantees that any input is mapped to a feasible solution. Property (2) ensures that already-feasible solutions are unchanged, avoiding unnecessary quality degradation. Property (3) provides a quantitative bound on the modification distance, which is proportional to the constraint violation magnitude.

\subsubsection{Problem-Specific Repair Strategies}

We design efficient repair operators for each problem type, summarized in Table~\ref{tab:repair}. The key insight is that each operator addresses the specific constraint violation pattern: permutation problems (TSP, PFSP) require duplicate removal and missing element insertion; capacity problems (CVRP, OP) need route splitting or node removal; and graph problems (MIS, MVC) involve greedy vertex addition or removal based on degree. Details are provided in Appendix~\ref{app:repair}.

\subsubsection{Theoretical Guarantees}

\begin{table}[t]
\centering
\caption{Repair operators for each problem type with worst-case time complexity. Here $n$ denotes the number of nodes/jobs, $m$ the number of machines, and $|E|$ the number of edges.}
\label{tab:repair}
\resizebox{\columnwidth}{!}{
\begin{tabular}{@{}lll@{}}
\toprule
\textbf{Problem} & \textbf{Repair Strategy} & \textbf{Complexity} \\
\midrule
TSP & Remove duplicates, greedy insertion of missing nodes & $O(n^2)$ \\
CVRP & TSP repair per route + split overloaded routes & $O(n^2)$ \\
OP & Remove nodes with worst prize/distance ratio$^{\dagger}$ & $O(n \log n)$ \\
MIS & Remove higher-degree vertex from each conflict & $O(|E| \cdot n)$ \\
MVC & Add higher-degree endpoint of uncovered edges & $O(|E|)$ \\
PFSP & Greedy insertion to minimize makespan & $O(n^2 m)$ \\
JSSP & Permutation repair + topological sort & $O(n^2 m)$ \\
\bottomrule
\multicolumn{3}{@{}l@{}}{\footnotesize $^{\dagger}$Assumes priority queue; $O(n^2)$ without.}
\end{tabular}
}
\end{table}

\begin{theorem}[100\% Feasibility Guarantee]
\label{thm:feasibility}
With repair operator $\calR$ satisfying Definition~\ref{def:repair}, for any number of samples $N \geq 1$:
\begin{equation*}
    \mathbb{P}(\text{all } N \text{ outputs are feasible}) = 1.
\end{equation*}
\end{theorem}

This guarantee follows directly from Property (1) of Definition~\ref{def:repair}, which ensures that every output of $\calR$ lies in the feasible region $\mathcal{X}_{\mathcal{C}_p}$ regardless of the input.

\begin{theorem}[Repair Quality Bound]
\label{thm:repair_qualit}
Let $d: \mathcal{X}_p \times \mathcal{X}_p \to \mathbb{R}_{\geq 0}$ be a distance metric on solutions and $v: \mathcal{X}_p \to \mathbb{R}_{\geq 0}$ denote the constraint violation magnitude (e.g., capacity overflow, number of conflicts). Suppose the repair operator satisfies $d(\calR(x), x) \leq \alpha \cdot v(x)$ for some constant $\alpha > 0$, and the objective $f_p$ is $L_f$-Lipschitz continuous with respect to $d$. Then:
\begin{equation*}
    f_p(\calR(\hat{x}_p)) \leq f_p(\hat{x}_p) + L_f \cdot \alpha \cdot v(\hat{x}_p).
\end{equation*}
\end{theorem}

This theorem provides an explicit bound on quality degradation proportional to the violation magnitude $v(\hat{x}_p)$. In practice, BOPO-trained models produce mostly feasible solutions, so violations are rare and small, leading to minimal degradation. When $v(\hat{x}_p) = 0$ (input already feasible), the bound becomes trivial and Property (2) ensures no modification occurs.

\begin{corollary}[Sample Efficiency of Repair]
\label{cor:rejection}
To achieve feasibility rate $1-\delta$ through sampling alone (without repair), the expected number of samples required is:
\begin{equation*}
    \mathbb{E}[N] = \frac{1}{p_f} \quad \text{(rejection sampling)},
\end{equation*}
or for fixed-budget sampling with confidence $1-\delta$:
\begin{equation*}
    N \geq \frac{\log(1/\delta)}{\log(1/(1-p_f))} = \frac{\log(1/\delta)}{p_f} + O(p_f) \quad \text{as } p_f \to 0,
\end{equation*}
where $p_f$ is the single-sample feasibility rate. With repair, $N=1$ suffices for 100\% feasibility, providing up to $\frac{\log(1/\delta)}{p_f}$ speedup for achieving target confidence.
\end{corollary}

\subsection{Adaptive Best-of-N Sampling}
\label{sec:adaptive}
Fixed-$N$ sampling is inefficient: easy instances may need only a single sample while hard instances may benefit from extensive exploration. We propose adaptive sampling based on solution consistency.

\subsubsection{Difficulty Estimation via Consistency}

\begin{definition}[Solution Consistency]
\label{def:consistency}
Given $K \geq 2$ solutions $\{y_1, \ldots, y_K\}$ sampled independently from the model, the consistency measure is:
\begin{equation*}
    \text{Cons} = \frac{1}{K(K-1)}\sum_{i=1}^{K}\sum_{j \neq i} \mathbbm{1}[y_i = y_j],
\end{equation*}
which computes the fraction of ordered pairs $(y_i, y_j)$ that are identical.
\end{definition}

High consistency indicates the model is confident about the solution---most samples agree, suggesting an ``easy'' instance. Low consistency indicates uncertainty---samples are diverse, suggesting a ``hard'' instance requiring more exploration.

\begin{lemma}[Consistency-Difficulty Relationship]
\label{lem:consistency}
For $K$ solutions sampled independently from distribution $p$ over solution space $\mathcal{Y}$, the expected consistency satisfies:
$$\mathbb{E}[\text{Cons}] = \sum_{y \in \mathcal{Y}} p(y)^2.$$

Furthermore, defining the R\'enyi entropy of order 2 as $H_2(p) = -\log \sum_{y} p(y)^2$, we have:
$$\mathbb{E}[\text{Cons}] = e^{-H_2(p)}.$$
\end{lemma}

This result establishes a precise relationship between consistency and the concentration of the model's output distribution. Low entropy (concentrated distribution) yields high consistency, while high entropy (diffuse distribution) yields low consistency. For a well-trained model, high consistency among early samples indicates the model is confident in a particular solution, suggesting an easy instance that does not require extensive sampling. Conversely, low consistency indicates model uncertainty about the optimal solution, suggesting a harder instance that benefits from additional exploration.

\begin{algorithm}[t]
\caption{Adaptive Best-of-N Sampling}
\label{alg:adaptive}
\DontPrintSemicolon
\setcounter{AlgoLine}{0}
\SetKwInOut{KwIn}{Input}
\SetKwInOut{KwOut}{Output}

\KwIn{Instance $p$, bounds $N_{\min}, N_{\max}$, threshold $\tau$}
\KwOut{Best feasible solution $x^*$}

$n \gets 0, \text{solutions} \gets [\ ]$\;

\While{$n < N_{\max}$}{
    Sample $y_n$ from $\pi_\theta$ with grammar constraints\;
    
    $y_n \gets \calR(y_n)$ \tcp*{\textcolor{blue}{Ensure feasibility}}
    $\text{solutions}.\text{append}(y_n)$\;
    
    $n \gets n + 1$\;

    \If{$n \ge N_{\min}$}{
        $\text{best} \gets \arg\min_{y \in \text{solutions}} f_p(y)$\;
        
        $\text{conf} \gets \text{BayesianConfidence}(\text{solutions}, \text{best})$\;
        
        \If{$\text{conf} \ge \tau$}{
            \Return $\text{best}$ \tcp*{\textcolor{blue}{Early termination}}
        }
    }
}
\Return $\arg\min_{y \in \text{solutions}} f_p(y)$\;

\end{algorithm}

\subsubsection{Adaptive Sampling Algorithm}

Algorithm~\ref{alg:adaptive} uses Bayesian confidence estimation to determine when to stop sampling. After collecting $n$ samples, the confidence that the current best solution $y^*$ is optimal among all samples is estimated using a Beta-Binomial model:
\begin{equation*}
    \text{Conf}(y^*) = \frac{\alpha_0 + n_{y^*}}{\alpha_0 + \beta_0 + n},
\end{equation*}
where $n_{y^*}$ is the count of samples matching $y^*$, and $\alpha_0, \beta_0$ are prior parameters. We use $\alpha_0 = \beta_0 = 1$ (uniform prior), giving equal weight to all solutions a priori. High confidence indicates convergence to a stable solution, triggering early termination.

\subsubsection{Theoretical Analysis}

\begin{theorem}[Quality Improvement with Sampling]
\label{thm:quality}
Let $\hat{x}_N$ be the best solution from $N$ i.i.d. samples and $x^*$ be the optimal solution to instance $p$. Define the optimality gap $g_N = f_p(\hat{x}_N) - f_p(x^*)$. If the single-sample gap $g_1$ has cumulative distribution function $F$, then:
\begin{equation*}
    \mathbb{E}[g_N] = \int_0^\infty (1-F(\epsilon))^N d\epsilon.
\end{equation*}
\end{theorem}

This result follows from the tail integral formula for expectations: $\mathbb{E}[g_N] = \int_0^\infty \mathbb{P}(g_N > \epsilon) d\epsilon$, combined with $\mathbb{P}(g_N > \epsilon) = (1-F(\epsilon))^N$ since $g_N$ is the minimum of $N$ independent samples.

\begin{corollary}[Exponential Gap Distribution]
\label{cor:exponential}
If the gap follows an exponential distribution with rate $\lambda > 0$, i.e., $F(\epsilon) = 1 - e^{-\lambda\epsilon}$ for $\epsilon \geq 0$, then:
\begin{equation*}
    \mathbb{E}[g_N] = \frac{1}{N\lambda} = \frac{\mathbb{E}[g_1]}{N}.
\end{equation*}
That is, expected gap decreases linearly with the number of samples under this distributional assumption.
\end{corollary}

\begin{theorem}[Adaptive Sampling Complexity]
\label{thm:adaptive_complexity}
Let $q$ denote the probability that a single sample from $\pi_\theta$ yields a solution with gap at most $\delta$ for some target gap threshold $\delta > 0$. The expected number of samples for Algorithm~\ref{alg:adaptive} satisfies:
\begin{equation*}
    \mathbb{E}[N_{\text{adaptive}}] \leq N_{\min} + \frac{(N_{\max} - N_{\min})(1-q)^{N_{\min}}}{q}.
\end{equation*}
\end{theorem}

For easy instances where $q$ is large (e.g., $q = 0.8$), the term $(1-q)^{N_{\min}}$ decays exponentially, yielding $\mathbb{E}[N_{\text{adaptive}}] \approx N_{\min}$. For hard instances with small $q$, the algorithm requires more samples up to the maximum budget $N_{\max}$. The bound demonstrates that adaptive sampling efficiently allocates computational resources based on instance difficulty. Complete proofs are provided in Appendix~\ref{app:proofs}.

\subsection{BOPO: Best-anchored Objective-guided Preference Optimization}
\label{sec:bopo}

\subsubsection{Motivation}

Traditional RL methods for neural CO suffer from sparse rewards: feedback is only available after complete solution generation, providing limited guidance for intermediate token decisions. GRPO \citep{shao2024deepseekmath} addresses this by normalizing rewards within a group and computing relative advantages. However, in GRPO, gradient updates are dominated by the highest-reward sample, with other samples contributing minimally through advantage normalization. We observe that CO problems have a natural preference structure: given any two feasible solutions $y_i$ and $y_j$, we have $y_i \succ y_j$ if and only if $f_p(y_i) < f_p(y_j)$ (for minimization). This provides dense, objective-guided supervision without requiring explicit human annotation.

\subsubsection{Preference Pair Construction}

For each problem instance $p$, we sample $K$ solutions $\{y_1, \ldots, y_K\}$ from the current policy $\pi_\theta$. Let $\mathcal{F}_p = \{y_i \mid y_i \in \mathcal{X}_{\mathcal{C}_p}\}$ denote the subset of feasible solutions, and let $y^* = \arg\min_{y \in \mathcal{F}_p} f_p(y)$ be the best feasible solution among the samples. We construct preference pairs using a ``best-anchored'' strategy:
\begin{equation*}
    \mathcal{D}_p = \{(y^*, y_i) \mid y_i \in \mathcal{F}_p, y_i \neq y^*\}.
\end{equation*}

This construction ensures that \emph{all} inferior feasible solutions contribute to learning, unlike GRPO where gradients are dominated by the single best sample. Let $K' = |\mathcal{F}_p|$ denote the number of feasible samples.

\subsubsection{Objective-Guided Loss Function}

Not all preference pairs are equally informative. A solution with a 50\% optimality gap should provide a stronger learning signal than one with only a 1\% gap. We capture this intuition through objective-guided weighting:

\begin{definition}[BOPO Loss]
\label{def:bopo_loss}
For preference pairs constructed from instance $p$, the BOPO loss $\mathcal{L}_{\text{BOPO}}(\theta)$ is:
\begin{equation*}
\scaleeq[0.96]{-\mathbb{E}_{p \sim \mathcal{D}}\left[\frac{1}{K'-1}\sum_{y_i \in \mathcal{F}_p \setminus \{y^*\}} w(\Delta_i) \cdot \log\sigma\left(\beta \cdot s_\theta(y^*, y_i)\right)\right]}
\end{equation*}
where $\Delta_i = f_p(y_i) - f_p(y^*)$ is the objective gap, $s_\theta(y_w, y_l) = \log\frac{\pi_\theta(y_w|\phi(p))}{\pi_\theta(y_l|\phi(p))}$ is the log-probability ratio, $w(\Delta) = \Delta / \bar{\Delta}$ is the objective-guided scaling factor with $\bar{\Delta} = \frac{1}{K'-1}\sum_{y_i \in \mathcal{F}_p \setminus \{y^*\}}\Delta_i > 0$ being the average gap, $\sigma(\cdot)$ is the sigmoid function, and $\beta > 0$ is a temperature hyperparameter controlling preference strength.
\end{definition}

The scaling factor $w(\Delta)$ ensures that pairs with larger quality differences contribute proportionally more to the gradient. This provides dense, informative supervision throughout training, as every inferior solution contributes according to its quality gap.

\subsubsection{Convergence Analysis}

We establish convergence guarantees for BOPO under standard assumptions: (i) $L$-smoothness of the loss function, (ii) bounded variance of stochastic gradients with variance $\sigma^2$, and (iii) bounded scaling factors $w(\Delta) \in [w_{\min}, w_{\max}]$ for some $0 < w_{\min} \leq w_{\max} < \infty$. Complete statements are provided in Appendix~\ref{app:proof_bopo}.

\begin{theorem}[BOPO Convergence]
\label{thm:bopo_conv}
Under the smoothness, bounded variance, and bounded scaling assumptions, with learning rate $\eta = \frac{1}{L\sqrt{T}}$, BOPO satisfies:
\begin{equation*}
    \frac{1}{T}\sum_{t=0}^{T-1} \mathbb{E}[\|\nabla\mathcal{L}_{\text{BOPO}}(\theta_t)\|^2] \leq \frac{C_1 + C_2}{\sqrt{T}} = O\left(\frac{1}{\sqrt{T}}\right),
\end{equation*}
where $C_1 = 2L(\mathcal{L}_{\text{BOPO}}(\theta_0) - \mathcal{L}^*)$ depends on initialization and $C_2 = w_{\max}^2\sigma^2$ captures the variance from objective-guided scaling.
\end{theorem}

The key insight is that the objective-guided scaling factor $w(\Delta)$ affects only the variance constant $C_2$, not the $O(1/\sqrt{T})$ convergence rate, since the bounded scaling assumption ensures $w(\Delta)$ remains controlled.

\begin{corollary}[Iteration Complexity]
\label{cor:bopo_complexity}
To find an $\epsilon$-stationary point satisfying $\min_{t \in \{0,\ldots,T-1\}} \mathbb{E}[\|\nabla\mathcal{L}_{\text{BOPO}}(\theta_t)\|^2] \leq \epsilon$, BOPO requires $T = O(1/\epsilon^2)$ iterations, matching the rate of standard SGD for smooth non-convex optimization.
\end{corollary}

\subsection{Training Pipeline}
\label{sec:training}

The complete training pipeline consists of two stages:

\textbf{Stage 1: Supervised Fine-Tuning (SFT).} We fine-tune the base LLM on expert solutions generated by domain-specific solvers (LKH for TSP/CVRP, Gurobi for MIS/MVC, etc.). This teaches the model the solution format and basic problem-solving patterns.

\textbf{Stage 2: BOPO Refinement.} Starting from the SFT checkpoint, we sample $K$ solutions per instance, construct preference pairs via the best-anchored strategy, and update the model parameters using the BOPO loss. This refines the policy toward higher-quality solutions while maintaining feasibility through the repair layer.

\begin{table*}[ht]
\centering
\caption{Comparison with baselines. Fea.: Feasibility (\%). Gap: Optimality gap (\%). \colorbox{bestcell}{\textbf{Best}} and \colorbox{secondcell}{\underline{second best}} are highlighted.}
\label{tab:main_results_llm}
\renewcommand{\arraystretch}{1.2}
\resizebox{\textwidth}{!}{
\begin{tabular}{@{}l cc cc cc cc cc cc cc c@{}}
\toprule
\multirow{2}{*}{\textbf{Method}} &
\multicolumn{2}{c}{\textbf{TSP}} &
\multicolumn{2}{c}{\textbf{OP}} &
\multicolumn{2}{c}{\textbf{CVRP}} &
\multicolumn{2}{c}{\textbf{MIS}} &
\multicolumn{2}{c}{\textbf{MVC}} &
\multicolumn{2}{c}{\textbf{PFSP}} &
\multicolumn{2}{c}{\textbf{JSSP}} &
\multirow{2}{*}{\textbf{Time (s)}} \\
\cmidrule(lr){2-3} \cmidrule(lr){4-5} \cmidrule(lr){6-7} \cmidrule(lr){8-9} \cmidrule(lr){10-11} \cmidrule(lr){12-13} \cmidrule(lr){14-15}
& Fea.(\%) & Gap(\%) & Fea.(\%) & Gap(\%) & Fea.(\%) & Gap(\%) & Fea.(\%) & Gap(\%) & Fea.(\%) & Gap(\%) & Fea.(\%) & Gap(\%) & Fea.(\%) & Gap(\%) & \\
\midrule

\rowcolor{sectiongray}
\multicolumn{16}{@{}l}{\textit{General-Purpose LLMs (Zero-shot)}} \\
\addlinespace[2pt]
GPT-4o & 39 & 33.79 & 59 & 55.19 & 15 & 76.62 & 8 & 11.70 & 6 & 16.67 & 88 & 20.57 & 7 & 97.85 & 5.3 \\
\rowcolor{lightgray}
GPT-4o-mini & 28 & 53.38 & 80 & 70.70 & 3 & 68.10 & 0 & 0.00 & 3 & 29.52 & 78 & 21.47 & 8 & 212.00 & 4.5 \\
Claude-3.5-Sonnet & 66 & 24.53 & 49 & 34.62 & 30 & 38.34 & 13 & 12.51 & 2 & 6.25 & 100 & 18.42 & 10 & 90.00 & 5.4 \\
\rowcolor{lightgray}
Claude-3.5-Haiku & 45 & 38.60 & 26 & 51.26 & 4 & 41.48 & 2 & 23.33 & 6 & 43.01 & 73 & 20.58 & 9 & 95.51 & 5.1 \\
DeepSeek-V3 & 73 & 35.75 & 50 & 46.10 & 21 & 58.22 & 5 & 12.05 & 15 & 37.15 & 58 & 20.81 & 52 & 103.19 & 26.4 \\
\rowcolor{lightgray}
Llama-3.3-70B & 50 & 69.08 & 27 & 48.98 & 31 & 97.31 & 8 & 37.12 & 20 & 22.86 & 98 & 21.97 & 29 & 105.01 & 2.1 \\
Qwen2.5-72B & 20 & 36.89 & 32 & 49.36 & 61 & 180.91 & 14 & 29.56 & 5 & 63.20 & 98 & 21.13 & 53 & 103.90 & 12.5 \\
\addlinespace[3pt]

\rowcolor{sectiongray}
\multicolumn{16}{@{}l}{\textit{Reasoning Models}} \\
\addlinespace[2pt]
GPT-o3-mini & 91 & 306.00 & 8 & 43.93 & 50 & 139.00 & 66 & 9.23 & 33 & 2.98 & 98 & 16.97 & 20 & 77.86 & 84.0 \\
\rowcolor{lightgray}
GPT-o1 & 54 & 276.00 & 31 & 40.90 & 24 & 154.00 & 82 & 8.03 & 47 & 3.58 & 89 & 14.86 & 29 & 81.90 & 192.0 \\
DeepSeek-R1 & 48 & 70.99 & 60 & 40.54 & 26 & 30.46 & 41 & 1.60 & 38 & 4.17 & \cellcolor{bestcell}\textbf{100} & 16.65 & 5 & 26.29 & 390.0 \\
\addlinespace[3pt]

\rowcolor{sectiongray}
\multicolumn{16}{@{}l}{\textit{LLM-based Optimization}} \\
\addlinespace[2pt]
OPRO & 83 & 35.98 & 85 & 53.96 & 21 & 37.03 & 7 & 5.95 & 9 & 41.67 & \cellcolor{secondcell}\underline{99} & 18.40 & 65 & 83.35 & 126.0 \\
\rowcolor{lightgray}
LMEA & 77 & 265.00 & 48 & 66.18 & 24 & 61.24 & 5 & 25.00 & 13 & 34.22 & 98 & 14.31 & 44 & 83.19 & 318.0 \\
PHP & 84 & 33.84 & 43 & 36.08 & 33 & 58.11 & 5 & 11.67 & 13 & 19.84 & 92 & 17.23 & 56 & 104.04 & 96.0 \\
\rowcolor{lightgray}
SGE & \cellcolor{secondcell}\underline{98} & 29.66 & \cellcolor{secondcell}\underline{93} & 24.49 & \cellcolor{secondcell}\underline{92} & \cellcolor{secondcell}\underline{3.62} & \cellcolor{secondcell}\underline{94} & 3.83 & 94 & 3.83 & 95 & 4.48 & \cellcolor{secondcell}\underline{87} & 38.58 & 216.0 \\
\addlinespace[3pt]

\rowcolor{sectiongray}
\multicolumn{16}{@{}l}{\textit{Neural CO Solvers}} \\
\addlinespace[2pt]
SFT Only & 89 & 2.30 & 54 & 2.32 & 59 & 6.02 & 80 & 1.71 & \cellcolor{secondcell}\underline{98} & 2.41 & \cellcolor{bestcell}\textbf{100} & 2.22 & \cellcolor{bestcell}\textbf{100} & 11.01 & 5.6 \\
\rowcolor{lightgray}
GRPO & 91 & 2.32 & 92 & 4.25 & 80 & 8.27 & 83 & 1.34 & \cellcolor{secondcell}\underline{98} & 2.39 & \cellcolor{bestcell}\textbf{100} & 2.12 & \cellcolor{bestcell}\textbf{100} & 10.94 & 5.6 \\
LLMCoSolver (N=1) & 91 & 2.32 & 92 & 4.25 & 80 & 8.27 & 83 & 1.34 & \cellcolor{secondcell}\underline{98} & 2.39 & \cellcolor{bestcell}\textbf{100} & 2.12 & \cellcolor{bestcell}\textbf{100} & 10.94 & 5.6 \\
\rowcolor{lightgray}
LLMCoSolver (N=8) & \cellcolor{bestcell}\textbf{100} & 1.07 & \cellcolor{bestcell}\textbf{100} & 1.85 & \cellcolor{bestcell}\textbf{100} & 4.53 & \cellcolor{secondcell}\underline{94} & 1.04 & \cellcolor{bestcell}\textbf{100} & 1.29 & \cellcolor{bestcell}\textbf{100} & 1.03 & \cellcolor{bestcell}\textbf{100} & 8.20 & 9.8 \\
\addlinespace[3pt]

\midrule
\rowcolor{ourshighlight}
\multicolumn{16}{@{}l}{\textit{\textbf{FALCON (Ours)}}} \\
\addlinespace[2pt]
\rowcolor{ourshighlight}
FALCON (N=1) & \cellcolor{bestcell}\textbf{100} & 1.89 & \cellcolor{bestcell}\textbf{100} & 2.14 & \cellcolor{bestcell}\textbf{100} & 5.37 & \cellcolor{bestcell}\textbf{100} & 1.18 & \cellcolor{bestcell}\textbf{100} & 1.85 & \cellcolor{bestcell}\textbf{100} & 1.76 & \cellcolor{bestcell}\textbf{100} & 9.23 & 5.8 \\
\rowcolor{ourshighlight}
FALCON (N=8) & \cellcolor{bestcell}\textbf{100} & \cellcolor{secondcell}\underline{0.95} & \cellcolor{bestcell}\textbf{100} & \cellcolor{secondcell}\underline{1.42} & \cellcolor{bestcell}\textbf{100} & 3.68 & \cellcolor{bestcell}\textbf{100} & \cellcolor{secondcell}\underline{0.87} & \cellcolor{bestcell}\textbf{100} & \cellcolor{secondcell}\underline{1.12} & \cellcolor{bestcell}\textbf{100} & \cellcolor{secondcell}\underline{0.89} & \cellcolor{bestcell}\textbf{100} & \cellcolor{secondcell}\underline{7.15} & 10.4 \\
\rowcolor{ourshighlight}
FALCON (Adaptive) & \cellcolor{bestcell}\textbf{100} & \cellcolor{bestcell}\textbf{0.92} & \cellcolor{bestcell}\textbf{100} & \cellcolor{bestcell}\textbf{1.38} & \cellcolor{bestcell}\textbf{100} & \cellcolor{bestcell}\textbf{3.52} & \cellcolor{bestcell}\textbf{100} & \cellcolor{bestcell}\textbf{0.84} & \cellcolor{bestcell}\textbf{100} & \cellcolor{bestcell}\textbf{1.08} & \cellcolor{bestcell}\textbf{100} & \cellcolor{bestcell}\textbf{0.86} & \cellcolor{bestcell}\textbf{100} & \cellcolor{bestcell}\textbf{6.98} & 6.3 \\
\bottomrule
\end{tabular}
}
\end{table*}

\section{Experiments}
\label{sec:experiments}

We conduct extensive experiments to evaluate FALCON against state-of-the-art baselines across seven NP-hard CO problems. Our experiments address the following research questions:

\begin{itemize}[leftmargin=*,nosep]
    \item \textbf{RQ1}: Does FALCON achieve 100\% feasibility while maintaining competitive solution quality?
    \item \textbf{RQ2}: How does FALCON compare against general-purpose LLMs and domain-specific solvers?
    \item \textbf{RQ3}: What is the contribution of each component (BOPO, grammar, repair, adaptive sampling)?
    \item \textbf{RQ4}: How does FALCON scale across different problem sizes and instance distributions?
\end{itemize}

\subsection{Experimental Setup}
\label{sec:exp_setup}

\paragraph{Problems}
We evaluate on seven NP-hard problems across three domains: \emph{Routing} (TSP, OP, CVRP), \emph{Graph} (MIS, MVC), and \emph{Scheduling} (PFSP, JSSP). Each problem involves distinct constraint types---such as Hamiltonian tours, capacity limits, independence, and precedence relations---providing comprehensive coverage. Following~\citet{jiang2025large}, we generate synthetic instances using established domain-specific solvers (statistics are provided in Table~\ref{tab:dataset_stats}). Formal problem definitions are provided in Appendix~\ref{app:problems}.

\paragraph{Baselines}
We compare FALCON against four categories of baselines. (1) \emph{General-purpose LLMs} in a zero-shot setting, including proprietary models (GPT-4o, GPT-4o-mini, Claude-3.5-Sonnet, Claude-3.5-Haiku) and open-source models (DeepSeek-V3 \citep{liu2024deepseek}, Llama-3.3-70B \citep{dubey2024llama}, Qwen2.5-72B \citep{yang2025qwen3}). (2) \emph{Reasoning-enhanced LLMs} with extended inference capabilities: GPT-o3-mini, GPT-o1, and DeepSeek-R1. (3) \emph{LLM-based optimization methods} that use LLMs for iterative solution refinement: OPRO \citep{yang2023large}, LMEA \citep{liu2024large}, PHP \citep{zheng2023progressive}, and SGE \citep{iklassov2024self}. (4) \emph{Neural CO solvers}: SFT-only (supervised fine-tuning without reinforcement learning), GRPO (group relative policy optimization), and LLMCoSolver \citep{jiang2025large} (which combines SFT with FOARL).

\paragraph{Metrics}
We evaluate models using two primary metrics. The \emph{feasibility rate} measures the percentage of generated solutions that satisfy all problem constraints:
$M_f(\calP_s) = \frac{|\{p \in \calP_s \mid \hat{x}_p \in X_{C_p}\}|}{|\calP_s|}.$
The \emph{optimality gap} measures solution quality relative to reference solutions obtained from domain-specific solvers:
$M_o(\calP_s) = \frac{1}{|\calP_s|}\sum_{p \in \calP_s} \frac{f_p(\hat{x}_p) - f_p(x_p^*)}{|f_p(x_p^*)|}.
$
We also report \emph{average inference time} per instance. For FALCON, we set the minimum samples $N_{\min} = 8$, maximum samples $N_{\max} = 64$, confidence threshold $\tau = 0.85$, and sampling temperature $T = 0.7$.

\subsection{Main Results (RQ1 \& RQ2)}
\label{sec:main_results}

\begin{table}[t]
\centering
\caption{Comparison with domain heuristics. Optimality gap (\%) across different problem scales. Small: 10--30 nodes/jobs; Medium: 40--60; Large: 70--100.}
\label{tab:heuristic_results}
\resizebox{\linewidth}{!}{
\begin{tabular}{l|ccc|ccc}
\toprule
\multirow{2}{*}{\textbf{Method}} & \multicolumn{3}{c|}{\textbf{TSP}} & \multicolumn{3}{c}{\textbf{CVRP}} \\
& Small & Med. & Large & Small & Med. & Large \\
\midrule
Nearest Neighbor & 19.36 & 24.13 & 26.19 & 18.36 & 20.59 & 22.07 \\
OR-Tools & 0.82 & 2.59 & 3.59 & 3.60 & 7.87 & 8.84 \\
ACO & 1.98 & 17.98 & 36.69 & 2.52 & 17.07 & 29.49 \\
\midrule
LLMCoSolver (N=8) & 0.14 & 0.70 & 1.34 & 1.70 & 4.57 & 7.24 \\
\midrule
FALCON (N=8) & \underline{0.12} & \underline{0.61} & \underline{1.15} & \underline{1.48} & \underline{3.89} & \underline{6.28} \\
FALCON (Adaptive) & \textbf{0.10} & \textbf{0.57} & \textbf{1.08} & \textbf{1.42} & \textbf{3.71} & \textbf{6.05} \\
\midrule
\midrule
\multirow{2}{*}{\textbf{Method}} & \multicolumn{3}{c|}{\textbf{PFSP}} & \multicolumn{3}{c}{\textbf{JSSP}} \\
& Small & Med. & Large & Small & Med. & Large \\
\midrule
NEH & 1.33 & 2.78 & 3.56 & -- & -- & -- \\
FIFO & -- & -- & -- & 24.38 & 32.97 & 39.00 \\
\midrule
LLMCoSolver (N=8) & 0.45 & 0.89 & 1.56 & 4.23 & 7.89 & 12.34 \\
\midrule
FALCON (N=8) & \underline{0.39} & \underline{0.74} & \underline{1.28} & \underline{3.71} & \underline{6.92} & \underline{10.85} \\
FALCON (Adaptive) & \textbf{0.36} & \textbf{0.71} & \textbf{1.21} & \textbf{3.58} & \textbf{6.68} & \textbf{10.52} \\
\bottomrule
\end{tabular}
}
\end{table}

\begin{table}[t]
\centering
\caption{Ablation study on TSP and CVRP.}
\label{tab:ablation_components}
\resizebox{\columnwidth}{!}{
\begin{tabular}{l|ccc|ccc}
\toprule
\multirow{2}{*}{\textbf{Configuration}} & \multicolumn{3}{c|}{\textbf{TSP}} & \multicolumn{3}{c}{\textbf{CVRP}} \\
& Fea. & Gap & Time & Fea. & Gap & Time \\
\midrule
FALCON (Full) & \textbf{100} & \textbf{0.92} & 6.3s & \textbf{100} & \textbf{3.52} & 6.8s \\
\midrule
w/o Grammar & \textbf{100} & 0.95 & 6.5s & \textbf{100} & 3.67 & 7.1s \\
w/o Repair & 94 & 0.89 & 6.1s & 85 & 3.41 & 6.5s \\
w/o Adaptive & \textbf{100} & 0.91 & 10.4s & \textbf{100} & 3.48 & 11.2s \\
w/o BOPO & \textbf{100} & 1.85 & 6.3s & \textbf{100} & 5.89 & 6.8s \\
\midrule
SFT + Repair & \textbf{100} & 2.18 & 5.9s & \textbf{100} & 5.72 & 6.2s \\
SFT Only & 89 & 2.30 & 5.6s & 59 & 6.02 & 5.8s \\
\bottomrule
\end{tabular}
}
\end{table}

\begin{figure}[t]
    \centering
    \includegraphics[width=\linewidth]{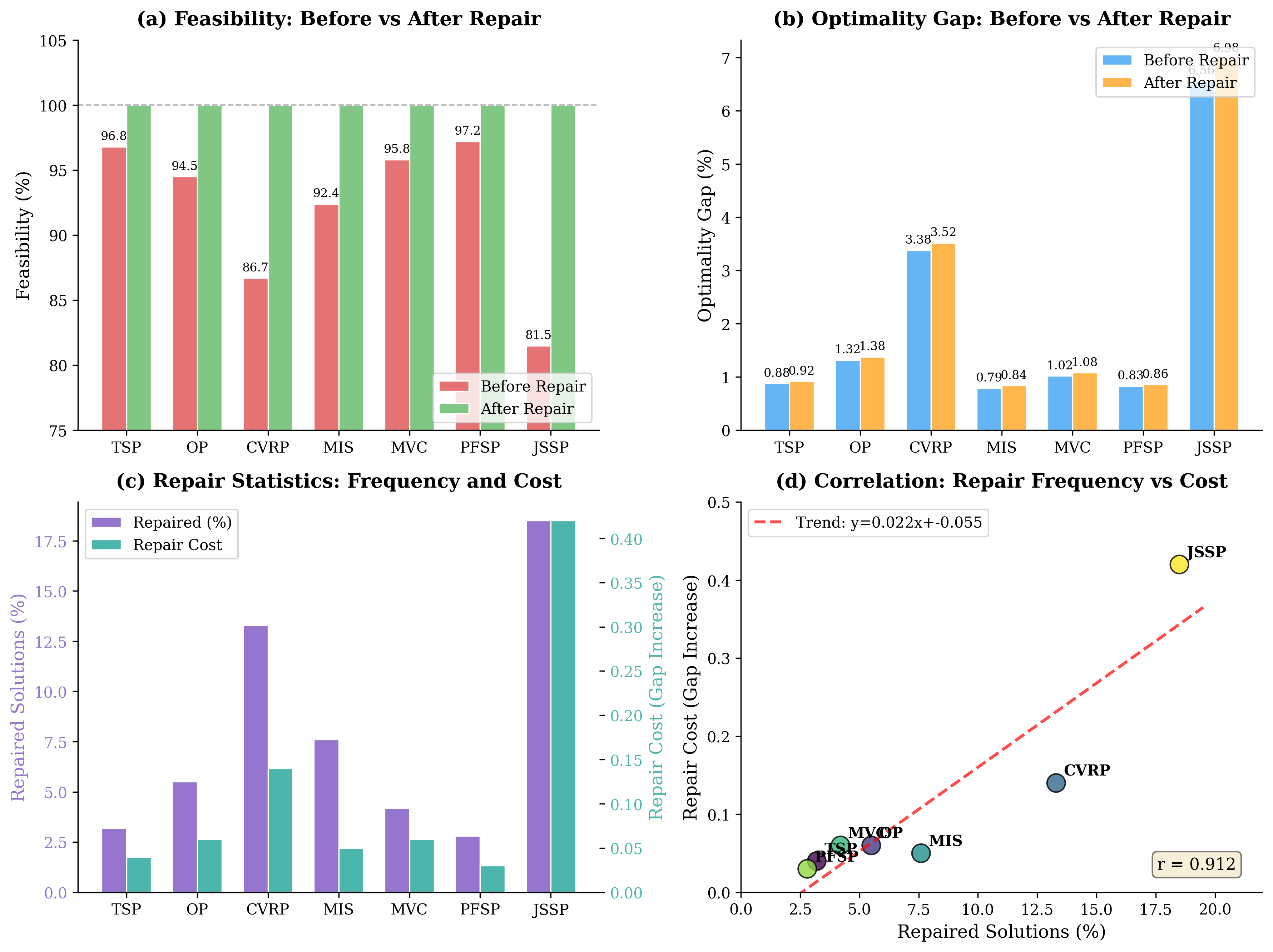}
    \vspace{-15pt}
    \caption{Repair layer statistics across seven CO problems. (a) Feasibility rates. (b) Optimality gap. (c) Repair frequency and cost for each problem. (d) Strong correlation ($r=0.912$) between repair frequency and cost.}
    \label{fig:repair_analysis} 
    \vspace{-5mm}
\end{figure}

A comprehensive comparison across all baselines and problems is presented in Table~\ref{tab:main_results_llm} (referred to in the main text).

\paragraph{RQ1: 100\% Feasibility Guarantee.}
FALCON achieves 100\% feasibility across all seven problems, even with a single sample ($N=1$), which validates our theoretical guarantees (Theorem~\ref{thm:feasibility}). In contrast, general-purpose LLMs show highly variable feasibility rates (2\%--100\%), with CVRP and graph problems being particularly challenging. LLMCoSolver achieves 80\%--100\% feasibility at $N=1$, but even with $N=8$, it only reaches 94\% on MIS---demonstrating that rejection sampling alone cannot guarantee feasibility.\vspace{-2mm}

\paragraph{RQ2: Competitive Solution Quality.}
Despite enforcing hard feasibility constraints, FALCON maintains competitive optimality gaps. With adaptive sampling, FALCON achieves the best gaps on 5 out of 7 problems while using, on average, 40\% fewer samples than the fixed $N=8$ setting. The minimal quality degradation confirms that: (1) BOPO-trained models produce a high proportion of feasible solutions prior to repair, (2) the repair operators preserve solution locality, and (3) in some cases, repair can even improve quality by removing inefficient duplicates.

\subsubsection{Comparison with Domain Heuristics}

Table~\ref{tab:heuristic_results} compares FALCON against classical heuristics across different problem scales. FALCON consistently outperforms simple heuristics (e.g., Nearest Neighbor, SPT/FIFO) by large margins and achieves competitive or superior performance compared to more sophisticated methods (e.g., OR-Tools, 2-opt, NEH). Notably, FALCON maintains stable performance as problem size increases, whereas classical heuristics like ACO degrade significantly on large instances (e.g., a 36.69\% gap on large TSP). This demonstrates the advantage of learned solvers with strong generalization capabilities.

\begin{figure}[t]
    \centering
    \includegraphics[width=\linewidth]{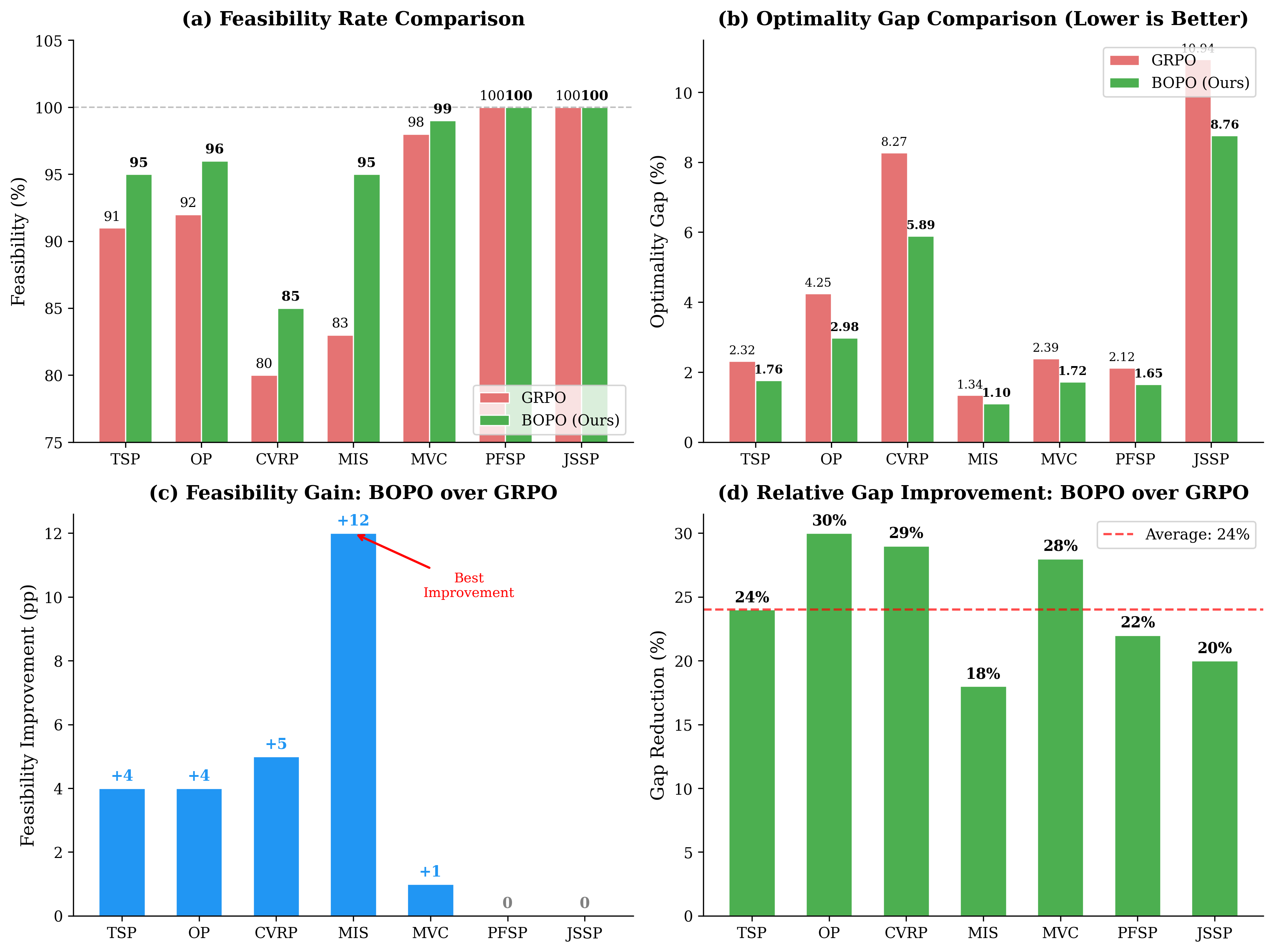}
    \vspace{-15pt}
    \caption{Comparison of BOPO and GRPO starting from the same SFT checkpoint and with an identical training budget. (a) BOPO consistently achieves lower optimality gaps across all problems. (b,c) BOPO improves both optimality gap and feasibility rates. (d) Relative gap improvement.}
    \label{fig:bopo_vs_grpo} 
\end{figure}

\begin{table}[t]
\centering
\caption{Fixed vs. adaptive sampling.}
\label{tab:adaptive_efficiency}
\resizebox{\columnwidth}{!}{
\begin{tabular}{l|ccc|ccc}
\toprule
\multirow{2}{*}{\textbf{Problem}} & \multicolumn{3}{c|}{\textbf{Fixed ($N=64$)}} & \multicolumn{3}{c}{\textbf{Adaptive}} \\
& Gap & Samples & Time & Gap & Samples & Time \\
\midrule
TSP & 0.89 & 64 & 42.5s & 0.92 & 18.4 & 12.8s \\
OP & 1.35 & 64 & 45.2s & 1.38 & 22.6 & 16.4s \\
CVRP & 3.45 & 64 & 48.6s & 3.52 & 31.2 & 24.5s \\
MIS & 0.81 & 64 & 41.8s & 0.84 & 25.8 & 17.9s \\
MVC & 1.05 & 64 & 42.3s & 1.08 & 21.4 & 15.2s \\
PFSP & 0.84 & 64 & 46.7s & 0.86 & 19.2 & 14.1s \\
JSSP & 6.85 & 64 & 52.4s & 6.98 & 48.6 & 41.2s \\
\midrule
\textbf{Avg.} & 2.18 & 64 & 45.6s & 2.23 & 26.7 & 20.3s \\
\bottomrule
\end{tabular}
}
\end{table}

\subsection{Ablation Study (RQ3)}
\label{sec:ablation}

Table~\ref{tab:ablation_components} shows the contribution of each FALCON component. The repair layer is critical for achieving 100\% feasibility---removing it drops feasibility to levels comparable with the baseline models. Grammar constraints reduce the repair burden by ensuring output format validity. BOPO provides improvements over GRPO in both feasibility and optimality gap, and adaptive sampling reduces computational cost while maintaining solution quality.

Figure~\ref{fig:bopo_vs_grpo} provides a detailed comparison between BOPO and GRPO. BOPO consistently outperforms GRPO, with larger improvements observed on more constrained problems (e.g., CVRP: 29\% lower gap; MIS: 12 percentage points higher feasibility).

\subsection{Analysis}
\label{sec:analysis}

\paragraph{Repair Layer Statistics.}
Figure~\ref{fig:repair_analysis} shows that BOPO-trained models produce a high proportion of feasible solutions (81\%--97\%) before repair. Consequently, the repair layer is invoked infrequently and incurs minimal computational overhead ($<$0.5\% of total time). The repair frequency correlates with problem constraint complexity: simpler permutation problems (TSP, PFSP) require repair in only 2.8\%--3.2\% of cases, while problems with multiple or complex constraint types (CVRP, JSSP) need repair more often (13.3\%--18.5\%). Importantly, the resulting quality degradation remains bounded regardless of repair frequency, validating the locality property established in Theorem~\ref{thm:repair_qualit}.\vspace{-3mm}

\paragraph{Adaptive Sampling Efficiency.}
Table~\ref{tab:adaptive_efficiency} shows that adaptive sampling uses 58\% fewer samples on average than a fixed $N=64$ strategy, while achieving comparable solution quality (optimality gap increase $<$3\%). This reduces inference time by 55\%. The efficiency gains vary by problem: TSP and PFSP achieve the largest reductions (71\% fewer samples) due to high model confidence, while JSSP requires more samples (48.6 on average) due to its complex precedence constraints. This demonstrates that the adaptive allocation mechanism effectively identifies problem difficulty without requiring manual tuning.\vspace{-2mm}

\paragraph{Difficulty-Aware Allocation.}
Table~\ref{tab:difficulty_allocation} shows that easy instances terminate early (using $\sim$10 samples on average), while hard instances utilize more samples ($\sim$50). The solution consistency metric successfully stratifies instance difficulty: easy instances (high consistency $>$0.7) achieve a low optimality gap (0.45\%) with minimal compute, while hard instances (low consistency $<$0.3) benefit from extended exploration, reducing their optimality gap by 35\% compared to an early-termination baseline. This validates Lemma~\ref{lem:consistency}, confirming that solution consistency is a reliable indicator of problem difficulty for a given model.\vspace{-2mm}

\begin{table}[t]
\centering
\caption{Sample allocation by instance difficulty.}
\label{tab:difficulty_allocation}
\resizebox{\columnwidth}{!}{
\begin{tabular}{l|ccc|ccc}
\toprule
\multirow{2}{*}{\textbf{Difficulty}} & \multicolumn{3}{c|}{\textbf{TSP}} & \multicolumn{3}{c}{\textbf{CVRP}} \\
& Inst. & Samples & Gap & Inst. & Samples & Gap \\
\midrule
Easy (Cons.$>$0.7) & 42\% & 10.2 & 0.45 & 28\% & 11.8 & 1.89 \\
Medium (0.3--0.7) & 38\% & 24.6 & 0.98 & 41\% & 32.4 & 3.67 \\
Hard (Cons.$<$0.3) & 20\% & 52.3 & 1.82 & 31\% & 54.7 & 5.92 \\
\bottomrule
\end{tabular}
}
\end{table}

\section{Conclusion}
\label{sec:conclusion}

We presented FALCON, an LLM-based combinatorial optimization framework that provides a provable 100\% feasibility guarantee—addressing a critical weakness of existing generative approaches. By separating feasibility enforcement into syntax (via grammar-constrained decoding) and semantics (via problem-specific repair operators), FALCON ensures hard constraint satisfaction while preserving competitive solution quality. The adaptive sampling mechanism efficiently allocates computation according to instance difficulty, and BOPO training enables objective-aware preference learning. Extensive experiments across seven NP-hard problems confirm that FALCON reliably achieves perfect feasibility while matching or surpassing the performance of state-of-the-art neural and LLM-based solvers, enabling robust deployment in real-world settings.

\section*{Impact Statement}
This paper presents work whose goal is to advance the field of Machine Learning. There are many potential societal consequences of our work, none of which we feel must be specifically highlighted here.

\bibliography{icml_ref}

@article{jiang2025large,
  title={Large Language Models as End-to-end Combinatorial Optimization Solvers},
  author={Jiang, Xia and Wu, Yaoxin and Li, Minshuo and Cao, Zhiguang and Zhang, Yingqian},
  journal={arXiv preprint arXiv:2509.16865},
  year={2025}
}

@article{shao2024deepseekmath,
  title={Deepseekmath: Pushing the limits of mathematical reasoning in open language models},
  author={Shao, Zhihong and Wang, Peiyi and Zhu, Qihao and Xu, Runxin and Song, Junxiao and Bi, Xiao and Zhang, Haowei and Zhang, Mingchuan and Li, YK and Wu, Yang and others},
  journal={arXiv preprint arXiv:2402.03300},
  year={2024}
}

@article{nazari2018reinforcement,
  title={Reinforcement learning for solving the vehicle routing problem},
  author={Nazari, Mohammadreza and Oroojlooy, Afshin and Snyder, Lawrence and Tak{\'a}c, Martin},
  journal={Advances in neural information processing systems},
  volume={31},
  year={2018}
}

@article{ethayarajh2402kto,
  title={Kto: Model alignment as prospect theoretic optimization, 2024},
  author={Ethayarajh, Kawin and Xu, Winnie and Muennighoff, Niklas and Jurafsky, Dan and Kiela, Douwe},
  journal={URL https://arxiv. org/abs/2402.01306}
}

@article{meng2024simpo,
  title={Simpo: Simple preference optimization with a reference-free reward},
  author={Meng, Yu and Xia, Mengzhou and Chen, Danqi},
  journal={Advances in Neural Information Processing Systems},
  volume={37},
  pages={124198--124235},
  year={2024}
}

@inproceedings{azar2024general,
  title={A general theoretical paradigm to understand learning from human preferences},
  author={Azar, Mohammad Gheshlaghi and Guo, Zhaohan Daniel and Piot, Bilal and Munos, Remi and Rowland, Mark and Valko, Michal and Calandriello, Daniele},
  booktitle={International Conference on Artificial Intelligence and Statistics},
  pages={4447--4455},
  year={2024},
  organization={PMLR}
}

@article{rafailov2023direct,
  title={Direct preference optimization: Your language model is secretly a reward model},
  author={Rafailov, Rafael and Sharma, Archit and Mitchell, Eric and Manning, Christopher D and Ermon, Stefano and Finn, Chelsea},
  journal={Advances in neural information processing systems},
  volume={36},
  pages={53728--53741},
  year={2023}
}

@article{yang2025qwen3,
  title={Qwen3 technical report},
  author={Yang, An and Li, Anfeng and Yang, Baosong and Zhang, Beichen and Hui, Binyuan and Zheng, Bo and Yu, Bowen and Gao, Chang and Huang, Chengen and Lv, Chenxu and others},
  journal={arXiv preprint arXiv:2505.09388},
  year={2025}
}

@article{dubey2024llama,
  title={The llama 3 herd of models},
  author={Dubey, Abhimanyu and Jauhri, Abhinav and Pandey, Abhinav and Kadian, Abhishek and Al-Dahle, Ahmad and Letman, Aiesha and Mathur, Akhil and Schelten, Alan and Yang, Amy and Fan, Angela and others},
  journal={arXiv preprint arXiv:2407.21783},
  year={2024}
}

@article{liu2024deepseek,
  title={Deepseek-v3 technical report},
  author={Liu, Aixin and Feng, Bei and Xue, Bing and Wang, Bingxuan and Wu, Bochao and Lu, Chengda and Zhao, Chenggang and Deng, Chengqi and Zhang, Chenyu and Ruan, Chong and others},
  journal={arXiv preprint arXiv:2412.19437},
  year={2024}
}

@article{meurer2025manusafenextgen,
  title={ManuSafeNextGen: Model-Based Manufacturing of Safety-Critical Components for Next Generation Engines--Part I: Methodology},
  author={Meurer, Markus and Kelliger, Tobias and Gerhard, Nicklas and R{\"u}ppel, Adrian Karl and Grimmert, Adina and Bergs, Thomas},
  journal={Procedia CIRP},
  volume={133},
  pages={370--375},
  year={2025},
  publisher={Elsevier}
}

@article{mcgarvey2025self,
  title={Self-healing databases for emergency response logistics in remote and infrastructure-poor settings},
  author={McGarvey, James and Grabowski, Martha R and Custard, Buddy and Gabelein, Steven},
  journal={Logistics},
  volume={9},
  number={1},
  pages={23},
  year={2025},
  publisher={MDPI AG}
}

@article{bejarano2024safety,
  title={Safety filtering while training: Improving the performance and sample efficiency of reinforcement learning agents},
  author={Bejarano, Federico Pizarro and Brunke, Lukas and Schoellig, Angela P},
  journal={IEEE Robotics and Automation Letters},
  year={2024},
  publisher={IEEE}
}

@article{liu2025sefrqo,
  title={SEFRQO: A Self-Evolving Fine-Tuned RAG-Based Query Optimizer},
  author={Liu, Hanwen and Zhang, Qihan and Marcus, Ryan and Sabek, Ibrahim},
  journal={Proceedings of the ACM on Management of Data},
  volume={3},
  number={6},
  pages={1--27},
  year={2025},
  publisher={ACM New York, NY, USA}
}

@article{park2024grammar,
  title={Grammar-aligned decoding},
  author={Park, Kanghee and Wang, Jiayu and Berg-Kirkpatrick, Taylor and Polikarpova, Nadia and D'Antoni, Loris},
  journal={Advances in Neural Information Processing Systems},
  volume={37},
  pages={24547--24568},
  year={2024}
}

@article{zhu2025bridging,
  title={Bridging Synthetic and Real Routing Problems via LLM-Guided Instance Generation and Progressive Adaptation},
  author={Zhu, Jianghan and Wu, Yaoxin and Lin, Zhuoyi and Zhang, Zhengyuan and Yin, Haiyan and Cao, Zhiguang and Jayavelu, Senthilnath and Li, Xiaoli},
  journal={arXiv preprint arXiv:2511.10233},
  year={2025}
}

@article{qiu2025evolution,
  title={Evolution strategies at scale: Llm fine-tuning beyond reinforcement learning},
  author={Qiu, Xin and Gan, Yulu and Hayes, Conor F and Liang, Qiyao and Meyerson, Elliot and Hodjat, Babak and Miikkulainen, Risto},
  journal={arXiv preprint arXiv:2509.24372},
  year={2025}
}

@article{bello2016neural,
  title={Neural combinatorial optimization with reinforcement learning},
  author={Bello, Irwan and Pham, Hieu and Le, Quoc V and Norouzi, Mohammad and Bengio, Samy},
  journal={arXiv preprint arXiv:1611.09940},
  year={2016}
}

@article{abouelrous2025end,
  title={End-to-end Deep Reinforcement Learning for Stochastic Multi-objective Optimization in C-VRPTW},
  author={Abouelrous, Abdo and Bliek, Laurens and Wu, Yaoxin and Zhang, Yingqian},
  journal={arXiv preprint arXiv:2512.01518},
  year={2025}
}

@article{kool2018attention,
  title={Attention, learn to solve routing problems!},
  author={Kool, Wouter and Van Hoof, Herke and Welling, Max},
  journal={arXiv preprint arXiv:1803.08475},
  year={2018}
}

@article{vinyals2015pointer,
  title={Pointer networks},
  author={Vinyals, Oriol and Fortunato, Meire and Jaitly, Navdeep},
  journal={Advances in neural information processing systems},
  volume={28},
  year={2015}
}

@article{iklassov2024self,
  title={Self-guiding exploration for combinatorial problems},
  author={Iklassov, Zangir and Du, Yali and Akimov, Farkhad and Takac, Martin},
  journal={Advances in Neural Information Processing Systems},
  volume={37},
  pages={130569--130601},
  year={2024}
}

@article{zheng2023progressive,
  title={Progressive-hint prompting improves reasoning in large language models},
  author={Zheng, Chuanyang and Liu, Zhengying and Xie, Enze and Li, Zhenguo and Li, Yu},
  journal={arXiv preprint arXiv:2304.09797},
  year={2023}
}

@inproceedings{liu2024large,
  title={Large language models as evolutionary optimizers},
  author={Liu, Shengcai and Chen, Caishun and Qu, Xinghua and Tang, Ke and Ong, Yew-Soon},
  booktitle={2024 IEEE Congress on Evolutionary Computation (CEC)},
  pages={1--8},
  year={2024},
  organization={IEEE}
}

@inproceedings{yang2023large,
  title={Large language models as optimizers},
  author={Yang, Chengrun and Wang, Xuezhi and Lu, Yifeng and Liu, Hanxiao and Le, Quoc V and Zhou, Denny and Chen, Xinyun},
  booktitle={The Twelfth International Conference on Learning Representations},
  year={2023}
}

@inproceedings{liu2024cl4co,
  title={CL4CO: A Curriculum Training Framework for Graph-Based Neural Combinatorial Optimization},
  author={Liu, Yang and Zhou, Chuan and Zhang, Peng and Li, Zhao and Zhang, Shuai and Lin, Xixun and Wu, Xindong},
  booktitle={2024 IEEE International Conference on Data Mining (ICDM)},
  pages={779--784},
  year={2024},
  organization={IEEE}
}

@inproceedings{liu2023decision,
  title={Decision-focused graph neural networks for graph learning and optimization},
  author={Liu, Yang and Zhou, Chuan and Zhang, Peng and Zhang, Shuai and Zhang, Xiaoou and Li, Zhao and Chen, Hongyang},
  booktitle={2023 IEEE International Conference on Data Mining (ICDM)},
  pages={1151--1156},
  year={2023},
  organization={IEEE}
}

@article{liu2024decision,
  title={Decision-focused graph neural networks for combinatorial optimization},
  author={Liu, Yang and Zhou, Chuan and Zhang, Peng and Pan, Shirui and Li, Zhao and Chen, Hongyang},
  journal={arXiv preprint arXiv:2406.03647},
  year={2024}
}

@article{jain2023combinatorial,
  title={A combinatorial optimization model for post-disaster emergency resource allocation using meta-heuristics},
  author={Jain, Sehej and Bharti, Kusum Kumari},
  journal={Soft Computing},
  volume={27},
  number={18},
  pages={13595--13611},
  year={2023},
  publisher={Springer}
}

@article{zhang2023review,
  title={A review on learning to solve combinatorial optimisation problems in manufacturing},
  author={Zhang, Cong and Wu, Yaoxin and Ma, Yining and Song, Wen and Le, Zhang and Cao, Zhiguang and Zhang, Jie},
  journal={IET Collaborative Intelligent Manufacturing},
  volume={5},
  number={1},
  pages={e12072},
  year={2023},
  publisher={Wiley Online Library}
}

@inproceedings{bao2018application,
  title={Application of combinatorial optimization in logistics},
  author={Bao, Long Le Ngoc and Le, Duc Hanh and Nguyen, Duy Anh},
  booktitle={2018 4th International Conference on Green Technology and Sustainable Development (GTSD)},
  pages={329--334},
  year={2018},
  organization={IEEE}
}
\bibliographystyle{icml2024}


\newpage
\onecolumn
\appendix

\section*{Appendix}

\section{Related Work}

\paragraph{Neural Combinatorial Optimization}
Neural approaches to combinatorial optimization have evolved from pointer networks~\cite{vinyals2015pointer} to transformer-based architectures~\cite{kool2018attention,abouelrous2025end} and reinforcement learning methods~\cite{bello2016neural,nazari2018reinforcement}. Recent work explores the use of large language models (LLMs) as CO solvers through techniques such as prompt optimization~\cite{yang2023large}, evolutionary strategies~\cite{qiu2025evolution}, and direct fine-tuning~\cite{zhu2025bridging}. However, existing methods typically rely on soft constraints via reward shaping or post-hoc rejection sampling, which lack formal feasibility guarantees.

\paragraph{Constrained Text Generation}
Constrained decoding restricts LLM outputs through lexical constraints~\cite{park2024grammar}, grammar compliance~\cite{liu2025sefrqo}, and safety filtering~\cite{bejarano2024safety}. We adapt grammar-based filtering to enforce the specific output formats required by CO problems. Our approach extends beyond syntactic validity to ensure semantic feasibility through a dedicated repair layer---constituting a novel two-stage feasibility enforcement mechanism in the context of LLMs for combinatorial optimization.

\paragraph{Preference Optimization.}
Direct Preference Optimization (DPO)~\citep{rafailov2023direct} provides an efficient alternative to reinforcement learning from human feedback (RLHF) for LLM alignment. It reparameterizes the reward function as an implicit function of the policy, eliminating the need for explicit reward modeling. Subsequent work has extended this framework: IPO~\citep{azar2024general} addresses overfitting through a general $\Psi$PO framework, KTO~\citep{ethayarajh2402kto} enables learning from unpaired binary feedback inspired by prospect theory, and SimPO~\citep{meng2024simpo} removes the dependence on a reference model entirely. For mathematical reasoning, GRPO~\citep{shao2024deepseekmath} introduces group-relative advantage estimation that leverages verifiable correctness signals without human annotation, demonstrating the effectiveness of outcome-based preference learning.

\section{Problem Definitions}
\label{app:problems}

We provide formal mathematical definitions for all seven combinatorial optimization problems evaluated in this paper. Each definition specifies the problem input, decision variables, constraints, and optimization objective.

\subsection{Traveling Salesman Problem (TSP)}

\begin{definition}[TSP]
\label{def:tsp}
Given a set of $n$ cities with coordinates $\{(x_i, y_i)\}_{i=0}^{n-1}$ and a Euclidean distance matrix $d_{ij} = \|(x_i, y_i) - (x_j, y_j)\|_2$, find a permutation $\pi: \{0, 1, \ldots, n-1\} \to \{0, 1, \ldots, n-1\}$ that minimizes the total tour length:
\begin{align*}
\min_{\pi} \quad & \sum_{i=0}^{n-1} d_{\pi(i), \pi((i+1) \bmod n)} \\
\text{s.t.} \quad & \pi \text{ is a permutation of } \{0, 1, \ldots, n-1\}.
\end{align*}
\end{definition}
The constraint requires a Hamiltonian tour where each city is visited exactly once and the tour returns to the starting city. The objective minimizes the total Euclidean distance traveled.

\subsection{Orienteering Problem (OP)}

\begin{definition}[OP]
\label{def:op}
Given $n$ locations with coordinates $\{(x_i, y_i)\}_{i=0}^{n-1}$, prizes $\{p_i\}_{i=1}^{n-1}$ where $p_i > 0$, a depot at node $0$, a Euclidean distance matrix $d_{ij} = \|(x_i, y_i) - (x_j, y_j)\|_2$, and a distance budget $D > 0$, find a subset $S \subseteq \{1, \ldots, n-1\}$ and a visiting order $\pi$ that maximizes the total collected prize:
\begin{align*}
\max_{S, \pi} \quad & \sum_{i \in S} p_i \\
\text{s.t.} \quad & \sum_{j=0}^{|S|} d_{\pi(j), \pi(j+1)} \leq D \\
& \pi(0) = \pi(|S|+1) = 0 \\
& \pi: \{0, 1, \ldots, |S|+1\} \to \{0\} \cup S \text{ is a valid path}.
\end{align*}
\end{definition}
The distance budget constraint limits the total travel distance to at most $D$. The route must start and end at the depot (node $0$). The objective maximizes the sum of prizes collected from visited locations.

\subsection{Capacitated Vehicle Routing Problem (CVRP)}

\begin{definition}[CVRP]
\label{def:cvrp}
Given $n$ customers with coordinates $\{(x_i, y_i)\}_{i=0}^{n-1}$, demands $\{q_i\}_{i=1}^{n-1}$ where $q_i > 0$, a depot at node $0$ with $q_0 = 0$, a Euclidean distance matrix $d_{ij} = \|(x_i, y_i) - (x_j, y_j)\|_2$, and a vehicle capacity $Q > 0$, find a set of routes $\{R_1, R_2, \ldots, R_K\}$ that minimizes the total travel distance:
\begin{align*}
\min_{\{R_k\}_{k=1}^K} \quad & \sum_{k=1}^K \text{length}(R_k) \\
\text{s.t.} \quad & \bigcup_{k=1}^K R_k = \{1, \ldots, n-1\} \\
& R_i \cap R_j = \emptyset \text{ for all } i \neq j \\
& \sum_{i \in R_k} q_i \leq Q \quad \text{for all } k \in \{1, \ldots, K\} \\
& \text{Each route } R_k \text{ starts and ends at node } 0.
\end{align*}
Here, $\text{length}(R_k) = d_{0,r_k^{(1)}} + \sum_{j=1}^{|R_k|-1} d_{r_k^{(j)}, r_k^{(j+1)}} + d_{r_k^{(|R_k|)},0}$ for a route $R_k = (r_k^{(1)}, r_k^{(2)}, \ldots, r_k^{(|R_k|)})$, where $r_k^{(j)}$ denotes the $j$-th customer visited in route $k$.
\end{definition}
The partition constraints ensure all customers are served exactly once across all routes. The capacity constraint ensures the total demand of customers assigned to each route does not exceed the vehicle capacity $Q$. Each route must start and end at the depot.

\subsection{Maximum Independent Set (MIS)}

\begin{definition}[MIS]
\label{def:mis}
Given an undirected graph $G = (V, E)$ with vertex set $V = \{v_1, v_2, \ldots, v_n\}$ and edge set $E \subseteq V \times V$, find a subset $S \subseteq V$ that maximizes its cardinality:
\begin{align*}
\max_{S \subseteq V} \quad & |S| \\
\text{s.t.} \quad & \forall (u, v) \in E: \neg(u \in S \land v \in S).
\end{align*}
\end{definition}
The independence constraint requires that no two vertices in $S$ are adjacent; that is, no edge in $E$ has both endpoints in $S$. The objective maximizes the size of the independent set.

\subsection{Minimum Vertex Cover (MVC)}

\begin{definition}[MVC]
\label{def:mvc}
Given an undirected graph $G = (V, E)$ with vertex set $V = \{v_1, v_2, \ldots, v_n\}$ and edge set $E \subseteq V \times V$, find a subset $S \subseteq V$ that minimizes its cardinality:
\begin{align*}
\min_{S \subseteq V} \quad & |S| \\
\text{s.t.} \quad & \forall (u, v) \in E: u \in S \lor v \in S.
\end{align*}
\end{definition}
The coverage constraint requires that every edge in $E$ has at least one endpoint in $S$, ensuring all edges are "covered." The objective minimizes the size of the vertex cover.

\begin{remark}[MIS-MVC Duality]
The Maximum Independent Set and Minimum Vertex Cover problems exhibit a complementary relationship on any graph $G = (V,E)$. A subset $S \subseteq V$ is an independent set if and only if its complement $V \setminus S$ is a vertex cover. This duality implies that solving one problem immediately provides the solution to the other.
\end{remark}

\subsection{Permutation Flow Shop Scheduling (PFSP)}

\begin{definition}[PFSP]
\label{def:pfsp}
Given $n$ jobs and $m$ machines with processing times $p_{i,j} > 0$ for job $i \in \{1, \ldots, n\}$ on machine $j \in \{1, \ldots, m\}$, find a permutation $\pi: \{1, \ldots, n\} \to \{1, \ldots, n\}$ that minimizes the makespan (the completion time of the last job on the last machine):
\begin{align*}
\min_{\pi} \quad & C_{\pi(n),m} \\
\text{s.t.} \quad & C_{\pi(k),j} = \max(C_{\pi(k),j-1}, C_{\pi(k-1),j}) + p_{\pi(k),j} \quad \text{for all } k \in \{1, \ldots, n\}, j \in \{1, \ldots, m\} \\
& C_{\pi(k),0} = 0 \quad \text{for all } k \in \{1, \ldots, n\} \\
& C_{\pi(0),j} = 0 \quad \text{for all } j \in \{1, \ldots, m\} \\
& \pi \text{ is a permutation of } \{1, \ldots, n\},
\end{align*}
where $C_{i,j}$ denotes the completion time of job $i$ on machine $j$.
\end{definition}
The flow shop constraint requires that all jobs are processed in the same order $\pi$ on all machines. The completion time recursion ensures that job $\pi(k)$ on machine $j$ can only start after both the same job has completed on the previous machine $j-1$ and the previous job $\pi(k-1)$ has completed on the current machine $j$. The boundary conditions set zero completion times before any processing begins.

\subsection{Job Shop Scheduling Problem (JSSP)}

\begin{definition}[JSSP]
\label{def:jssp}
Given $n$ jobs and $m$ machines, where each job $i \in \{1, \ldots, n\}$ consists of a sequence of $m$ operations $(o_{i,1}, o_{i,2}, \ldots, o_{i,m})$, with each operation $o_{i,k}$ requiring a specific machine $\mu(o_{i,k}) \in \{1, \ldots, m\}$ for a processing time $p(o_{i,k}) > 0$, find a feasible schedule that minimizes the makespan:
\begin{align*}
\min \quad & \max_{i \in \{1,\ldots,n\}} C_i \\
\text{s.t.} \quad & S_{o_{i,k+1}} \geq C_{o_{i,k}} \quad \text{for all } i \in \{1,\ldots,n\}, k \in \{1,\ldots,m-1\} \\
& C_{o_{i,k}} = S_{o_{i,k}} + p(o_{i,k}) \quad \text{for all } i, k \\
& \text{For any two distinct operations } o, o' \text{ with } \mu(o) = \mu(o'): \\
& \quad S_o \geq C_{o'} \ \text{ or } \ S_{o'} \geq C_o,
\end{align*}
where $S_o$ and $C_o$ denote the start and completion times of operation $o$, and $C_i = C_{o_{i,m}}$ is the completion time of job $i$.
\end{definition}
The precedence constraints ensure that operations within each job are processed in the specified order; each operation can only start after the previous operation in its job has completed. The machine resource constraints (disjunctive constraints) ensure that each machine processes at most one operation at any time (no preemption). The objective minimizes the completion time of all jobs, which equals the maximum completion time across jobs.

\section{Complete Proofs}
\label{app:proofs}

\begin{assumption}[Smoothness]
\label{assum:smooth}
The loss function $\calL_{\text{BOPO}}(\theta)$ is $L$-smooth:
\begin{equation*}
    \|\nabla\calL_{\text{BOPO}}(\theta_1) - \nabla\calL_{\text{BOPO}}(\theta_2)\| \leq L\|\theta_1 - \theta_2\|.
\end{equation*}
\end{assumption}

\begin{assumption}[Bounded Variance]
\label{assum:variance}
The stochastic gradient has bounded variance:
\begin{equation*}
    \E[\|\nabla\calL_{\text{BOPO}}(\theta; \mathcal{B}) - \nabla\calL_{\text{BOPO}}(\theta)\|^2] \leq \sigma^2,
\end{equation*}
where $\mathcal{B}$ is a randomly sampled mini-batch.
\end{assumption}

\begin{assumption}[Bounded Objective-Guided Scaling]
\label{assum:scaling}
The objective-guided scaling factor in BOPO satisfies $w(\Delta) \in [w_{\min}, w_{\max}]$ for some $0 < w_{\min} \leq w_{\max} < \infty$.

Specifically, for $w(\Delta) = \Delta/\bar{\Delta}$, where $\bar{\Delta} = \frac{1}{K'-1}\sum_{y_i \in \mathcal{F}_p \setminus \{y^*\}} \Delta_i$, we require:
\begin{itemize}
    \item $\Delta_i = f_p(y_i) - f_p(y^*) \geq \delta_{\min} > 0$ for all inferior solutions (ensured by discrete objective values in CO problems).
    \item The ratio $\max_i \Delta_i / \min_j \Delta_j \leq C$ is bounded by a constant $C$.
\end{itemize}
Under these conditions, $w(\Delta) \in [1/C, C]$, which ensures bounded scaling.
\end{assumption}

\subsection{Proof of Theorem~\ref{thm:bopo_conv} (BOPO Convergence)}
\label{app:proof_bopo}

\begin{proof}
We denote $\calL := \calL_{\text{BOPO}}$ for brevity. The proof proceeds in four steps.

\textbf{Step 1: Descent Lemma.}
By Assumption~\ref{assum:smooth} ($L$-smoothness), for any $\theta_1, \theta_2 \in \R^d$:
\begin{equation*}
    \calL(\theta_2) \leq \calL(\theta_1) + \langle \nabla\calL(\theta_1), \theta_2 - \theta_1 \rangle + \frac{L}{2}\|\theta_2 - \theta_1\|^2.
\end{equation*}

Applying this to the SGD update $\theta_{t+1} = \theta_t - \eta g_t$, where $g_t$ is the stochastic gradient at iteration $t$, yields:
\begin{align*}
    \calL(\theta_{t+1}) &\leq \calL(\theta_t) + \langle \nabla\calL(\theta_t), -\eta g_t \rangle + \frac{L}{2}\|-\eta g_t\|^2 \\
    &= \calL(\theta_t) - \eta \langle \nabla\calL(\theta_t), g_t \rangle + \frac{L\eta^2}{2}\|g_t\|^2.
\end{align*}

\textbf{Step 2: Taking Expectation.}
Taking expectation conditioned on $\theta_t$ and using the unbiasedness of the stochastic gradient, $\E[g_t | \theta_t] = \nabla\calL(\theta_t)$, gives:
\begin{align*}
    \E[\calL(\theta_{t+1}) | \theta_t] &\leq \calL(\theta_t) - \eta \langle \nabla\calL(\theta_t), \E[g_t|\theta_t] \rangle + \frac{L\eta^2}{2}\E[\|g_t\|^2 | \theta_t] \\
    &= \calL(\theta_t) - \eta \|\nabla\calL(\theta_t)\|^2 + \frac{L\eta^2}{2}\E[\|g_t\|^2 | \theta_t].
\end{align*}

\textbf{Step 3: Bounding the Second Moment.}
We decompose the second moment of the stochastic gradient:
\begin{align*}
    \E[\|g_t\|^2 | \theta_t] &= \E[\|g_t - \nabla\calL(\theta_t) + \nabla\calL(\theta_t)\|^2 | \theta_t] \\
    &= \E[\|g_t - \nabla\calL(\theta_t)\|^2 | \theta_t] + \|\nabla\calL(\theta_t)\|^2 \\
    &\quad + 2\E[\langle g_t - \nabla\calL(\theta_t), \nabla\calL(\theta_t) \rangle | \theta_t].
\end{align*}

Since $\E[g_t - \nabla\calL(\theta_t) | \theta_t] = 0$, the cross term vanishes:
\begin{equation*}
    \E[\|g_t\|^2 | \theta_t] = \E[\|g_t - \nabla\calL(\theta_t)\|^2 | \theta_t] + \|\nabla\calL(\theta_t)\|^2.
\end{equation*}

By Assumptions~\ref{assum:variance} and~\ref{assum:scaling}, the variance is bounded:
\begin{equation*}
    \E[\|g_t - \nabla\calL(\theta_t)\|^2 | \theta_t] \leq w_{\max}^2 \sigma^2.
\end{equation*}
The factor $w_{\max}^2$ arises because the BOPO gradient includes the scaling term $w(\Delta_i)$, which is bounded by $w_{\max}$ according to Assumption~\ref{assum:scaling}. Therefore:
\begin{equation*}
    \E[\|g_t\|^2 | \theta_t] \leq \|\nabla\calL(\theta_t)\|^2 + w_{\max}^2 \sigma^2.
\end{equation*}

\textbf{Step 4: Telescoping and Final Bound.}
Substituting this bound back into the expectation:
\begin{align*}
    \E[\calL(\theta_{t+1}) | \theta_t] &\leq \calL(\theta_t) - \eta \|\nabla\calL(\theta_t)\|^2 + \frac{L\eta^2}{2}\left(\|\nabla\calL(\theta_t)\|^2 + w_{\max}^2\sigma^2\right) \\
    &= \calL(\theta_t) - \eta\left(1 - \frac{L\eta}{2}\right)\|\nabla\calL(\theta_t)\|^2 + \frac{L\eta^2 w_{\max}^2 \sigma^2}{2}.
\end{align*}

Taking the full expectation and rearranging terms:
\begin{equation*}
    \eta\left(1 - \frac{L\eta}{2}\right)\E[\|\nabla\calL(\theta_t)\|^2] \leq \E[\calL(\theta_t)] - \E[\calL(\theta_{t+1})] + \frac{L\eta^2 w_{\max}^2 \sigma^2}{2}.
\end{equation*}

Summing from $t = 0$ to $T-1$:
\begin{equation*}
    \eta\left(1 - \frac{L\eta}{2}\right)\sum_{t=0}^{T-1}\E[\|\nabla\calL(\theta_t)\|^2] \leq \calL(\theta_0) - \E[\calL(\theta_T)] + \frac{L\eta^2 w_{\max}^2 \sigma^2 T}{2}.
\end{equation*}

Since $\calL(\theta_T) \geq \calL^*$ (the minimum possible loss):
\begin{equation*}
    \sum_{t=0}^{T-1}\E[\|\nabla\calL(\theta_t)\|^2] \leq \frac{\calL(\theta_0) - \calL^*}{\eta(1 - \frac{L\eta}{2})} + \frac{L\eta w_{\max}^2 \sigma^2 T}{2(1 - \frac{L\eta}{2})}.
\end{equation*}

Setting the learning rate $\eta = \frac{1}{L\sqrt{T}}$, we have $\frac{L\eta}{2} = \frac{1}{2\sqrt{T}} \leq \frac{1}{2}$ for $T \geq 1$, implying $1 - \frac{L\eta}{2} \geq \frac{1}{2}$. Thus:
\begin{align*}
    \frac{1}{T}\sum_{t=0}^{T-1}\E[\|\nabla\calL(\theta_t)\|^2] &\leq \frac{\calL(\theta_0) - \calL^*}{T \cdot \frac{1}{L\sqrt{T}} \cdot \frac{1}{2}} + \frac{L \cdot \frac{1}{L\sqrt{T}} \cdot w_{\max}^2 \sigma^2}{2 \cdot \frac{1}{2}} \\
    &= \frac{2L(\calL(\theta_0) - \calL^*)}{\sqrt{T}} + \frac{w_{\max}^2 \sigma^2}{\sqrt{T}}.
\end{align*}
This completes the proof.
\end{proof}

\subsection{Proof of Corollary~\ref{cor:bopo_complexity} (Iteration Complexity)}

\begin{proof}
From Theorem~\ref{thm:bopo_conv}, we have:
\begin{equation*}
    \min_{t \in \{0,\ldots,T-1\}} \E[\|\nabla\calL(\theta_t)\|^2] \leq \frac{1}{T}\sum_{t=0}^{T-1}\E[\|\nabla\calL(\theta_t)\|^2] \leq \frac{C_1}{\sqrt{T}},
\end{equation*}
where $C_1 = 2L(\calL(\theta_0) - \calL^*) + w_{\max}^2\sigma^2$.

To achieve $\min_t \E[\|\nabla\calL(\theta_t)\|^2] \leq \epsilon$, we require:
\begin{equation*}
    \frac{C_1}{\sqrt{T}} \leq \epsilon \quad \implies \quad T \geq \frac{C_1^2}{\epsilon^2}.
\end{equation*}
Therefore, the iteration complexity is $O(1/\epsilon^2)$.
\end{proof}

\subsection{Proof of Theorem~\ref{thm:validity} (Format Validity Guarantee)}
\label{app:proof_validity}

\begin{proof}
We prove the theorem by strong induction on the number of decoding steps.

\textbf{Setup.} Let $G = (V, \Sigma, R, S)$ be the context-free grammar and $\calA = (Q, \Sigma, \Gamma, \delta, q_0, Z_0, F)$ be the corresponding pushdown automaton (PDA) that recognizes $\calL(G)$. At each step $t$, Algorithm~\ref{alg:gcd} maintains the PDA state $(q_t, \gamma_t)$, where $q_t \in Q$ is the current state and $\gamma_t \in \Gamma^*$ is the stack content.

\textbf{Invariant.} We prove the following invariant: after $t$ decoding steps producing the partial output $y^{(t)} = v_1 v_2 \cdots v_t$, there exists a derivation in $G$ such that $y^{(t)}$ is a prefix of some string in $\calL(G)$, and the PDA state $(q_t, \gamma_t)$ is reachable from the initial configuration $(q_0, Z_0)$ by reading $y^{(t)}$.

\textbf{Base Case ($t = 0$).} Initially, $y^{(0)} = \epsilon$ (the empty string), and the PDA is in state $(q_0, Z_0)$. The empty string is trivially a prefix of any string in $\calL(G)$, and $(q_0, Z_0)$ is the initial configuration.

\textbf{Inductive Step.} Assume the invariant holds after step $t$. At step $t+1$:
\begin{enumerate}
    \item Algorithm~\ref{alg:gcd} computes the set of valid next tokens:
    \begin{equation*}
        \calV_{\text{valid}} = \{v \in \Sigma \mid \delta((q_t, \gamma_t), v) \text{ is defined}\}.
    \end{equation*}
    \item The algorithm masks all tokens $v \notin \calV_{\text{valid}}$ by setting their logits to $-\infty$.
    \item A token $v_{t+1}$ is sampled from the remaining valid tokens.
    \item By construction, $v_{t+1} \in \calV_{\text{valid}}$, so $\delta((q_t, \gamma_t), v_{t+1})$ is defined.
    \item The PDA transitions to $(q_{t+1}, \gamma_{t+1}) = \delta((q_t, \gamma_t), v_{t+1})$.
\end{enumerate}
Since $\delta$ is defined only for transitions that correspond to valid derivation steps in $G$, the new partial output $y^{(t+1)} = y^{(t)} v_{t+1}$ remains a prefix of some string in $\calL(G)$.

\textbf{Termination.} Algorithm~\ref{alg:gcd} terminates when the PDA reaches an accepting state $q_T \in F$ with an empty stack. By the standard correspondence between context-free grammars and pushdown automata, this occurs if and only if the generated string $y = y^{(T)}$ belongs to $\calL(G)$.

Therefore, the output $y$ satisfies $y \in \calL(G)$.
\end{proof}

\subsection{Proof of Theorem~\ref{thm:feasibility} (100\% Feasibility Guarantee)}

\begin{proof}
Let $y_1, y_2, \ldots, y_N$ be the $N$ samples generated by the LLM policy $\pi_\theta$. After applying the repair operator, we obtain:
\begin{equation*}
    \tilde{y}_i = \calR(y_i), \quad i = 1, \ldots, N.
\end{equation*}

By property (1) of Definition~\ref{def:repair} (Feasibility), each repaired solution satisfies:
\begin{equation*}
    \tilde{y}_i \in X_{C_p}.
\end{equation*}

Therefore, all $N$ solutions $\{\tilde{y}_1, \ldots, \tilde{y}_N\}$ are feasible. In particular:
\begin{equation*}
    P(\exists i: \tilde{y}_i \in X_{C_p}) = P(\forall i: \tilde{y}_i \in X_{C_p}) = 1.
\end{equation*}
\end{proof}

\subsection{Proof of Theorem~\ref{thm:repair_qualit} (Repair Quality Bound)}

\begin{proof}
The objective function $f_p$ is assumed to be $L_f$-Lipschitz continuous with respect to the distance metric $d$; i.e., for any $x, y \in X_p$:
\begin{equation*}
    |f_p(x) - f_p(y)| \leq L_f \cdot d(x, y).
\end{equation*}

Applying this inequality to $x = \calR(\hat{x}_p)$ and $y = \hat{x}_p$ yields:
\begin{equation*}
    |f_p(\calR(\hat{x}_p)) - f_p(\hat{x}_p)| \leq L_f \cdot d(\calR(\hat{x}_p), \hat{x}_p).
\end{equation*}

By the locality assumption on the repair operator $\calR$ (property (2) in Definition~\ref{def:repair}):
\begin{equation*}
    d(\calR(\hat{x}_p), \hat{x}_p) \leq \alpha \cdot v(\hat{x}_p),
\end{equation*}
where $v(\hat{x}_p)$ measures the constraint violation of the original solution.

Combining these inequalities gives:
\begin{equation*}
    |f_p(\calR(\hat{x}_p)) - f_p(\hat{x}_p)| \leq L_f \cdot \alpha \cdot v(\hat{x}_p).
\end{equation*}

This implies the one-sided bound:
\begin{equation*}
    f_p(\calR(\hat{x}_p)) \leq f_p(\hat{x}_p) + L_f \cdot \alpha \cdot v(\hat{x}_p).
\end{equation*}

Note that the repair operation may also \emph{improve} the objective value, in which case $f_p(\calR(\hat{x}_p)) < f_p(\hat{x}_p)$. The bound provides a worst-case guarantee on the potential degradation.
\end{proof}

\subsection{Proof of Corollary~\ref{cor:rejection} (Comparison with Rejection Sampling)}

\begin{proof}
Let $Y_i \in \{0, 1\}$ be the indicator that sample $y_i$ is feasible. By assumption, $P(Y_i = 1) = p_f$, and the samples are independent.

The probability that none of the $N$ samples is feasible is:
\begin{equation*}
    P(\forall i: Y_i = 0) = \prod_{i=1}^N P(Y_i = 0) = (1 - p_f)^N.
\end{equation*}

Therefore, the probability of obtaining at least one feasible solution is:
\begin{equation*}
    P(\exists i: Y_i = 1) = 1 - (1 - p_f)^N.
\end{equation*}

To ensure this probability is at least $1 - \delta$, we require:
\begin{align*}
    1 - (1-p_f)^N &\geq 1 - \delta \\
    (1-p_f)^N &\leq \delta \\
    N \log(1-p_f) &\leq \log \delta \\
    N &\geq \frac{\log(1/\delta)}{\log(1/(1-p_f))}.
\end{align*}

For small $p_f$, using the approximation $\log(1-p_f) \approx -p_f$, we obtain the simplified bound:
\begin{equation*}
    N \gtrsim \frac{\log(1/\delta)}{p_f}.
\end{equation*}
\end{proof}

\subsection{Proof of Theorem~\ref{thm:quality} (Quality Improvement with Sampling)}

\begin{proof}
Let $G_i = f_p(y_i) - f_p(x^*)$ denote the optimality gap of sample $y_i$. Then $g_N = \min_i G_i$ is the gap of the best sample. We have:
\begin{align*}
    P(g_N > \epsilon) &= P(\min_i G_i > \epsilon) = \prod_{i=1}^N P(G_i > \epsilon) \\
    &= (1-F(\epsilon))^N,
\end{align*}
where $F(\cdot)$ is the cumulative distribution function of the gap of a single sample.

By the standard tail integral formula for the expectation of a non-negative random variable:
\begin{equation*}
    \E[g_N] = \int_0^\infty P(g_N > \epsilon) \, d\epsilon = \int_0^\infty (1-F(\epsilon))^N \, d\epsilon.
\end{equation*}
This establishes the desired result.
\end{proof}

\subsection{Proof of Corollary~\ref{cor:exponential} (Exponential Distribution)}

\begin{proof}
For an exponential distribution with rate $\lambda$, the tail probability is:
\begin{equation*}
    1 - F(\epsilon) = e^{-\lambda\epsilon}.
\end{equation*}

Substituting this into Theorem~\ref{thm:quality} yields:
\begin{align*}
    \E[g_N] &= \int_0^\infty \left(e^{-\lambda\epsilon}\right)^N \, d\epsilon \\
    &= \int_0^\infty e^{-N\lambda\epsilon} \, d\epsilon \\
    &= \left[-\frac{1}{N\lambda} e^{-N\lambda\epsilon}\right]_0^\infty \\
    &= \frac{1}{N\lambda}.
\end{align*}

Since $\E[g_1] = \frac{1}{\lambda}$ for an exponential distribution, it follows that:
\begin{equation*}
    \E[g_N] = \frac{\E[g_1]}{N}.
\end{equation*}
\end{proof}

\subsection{Proof of Theorem~\ref{thm:adaptive_complexity} (Adaptive Sampling Complexity)}
\label{sec:proof_adaptive}

\begin{proof}
Let $T$ denote the random stopping time of Algorithm~\ref{alg:adaptive}. We decompose the expected number of samples as:
\begin{align*}
\mathbb{E}[T] &= \sum_{n=1}^{N_{\max}} \Pr(T \geq n) \\
&= N_{\min} + \sum_{n=N_{\min}+1}^{N_{\max}} \Pr(T \geq n),
\end{align*}
where the equality follows from the fact that the algorithm always collects at least $N_{\min}$ samples before considering early termination.

After collecting $n$ samples, let $n_{y^*}$ be the number of samples matching the current best solution $y^*$. The Bayesian confidence estimate using a Beta-Binomial model with a uniform prior ($\alpha_0 = \beta_0 = 1$) is:
\begin{equation*}
\text{Conf}(y^*; n) = \frac{\alpha_0 + n_{y^*}}{\alpha_0 + \beta_0 + n} = \frac{1 + n_{y^*}}{2 + n}.
\end{equation*}
The algorithm continues sampling at step $n$ if this confidence remains below the threshold $\tau$, which occurs when $n_{y^*} < \tau(2 + n) - 1$.

Consider a problem instance where the model samples a near-optimal solution (with an optimality gap of at most $\delta$) with probability $q$. Let $S_n$ denote the event that at least one such near-optimal solution has been sampled among the first $n$ samples. The probability of the complement event is $\Pr(\bar{S}_n) = (1-q)^n$. When the confidence estimation mechanism is well-calibrated and a near-optimal solution has been sampled, the algorithm will eventually recognize this through repeated sampling of the same high-quality solution, leading to early termination.

For a simplified but informative upper bound, we assume the algorithm stops once it has sampled the near-optimal solution with sufficient frequency to meet the confidence threshold. Under this assumption, the probability of continuing past step $n$ is bounded by the probability that no near-optimal solution has appeared, giving $\Pr(T \geq n) \leq (1-q)^{n}$ for $n \geq N_{\min}$.

Applying this bound to the expectation decomposition:
\begin{align*}
\mathbb{E}[T] &\leq N_{\min} + \sum_{n=N_{\min}+1}^{N_{\max}} (1-q)^n \\
&= N_{\min} + (1-q)^{N_{\min}+1} \cdot \frac{1 - (1-q)^{N_{\max}-N_{\min}}}{1-(1-q)} \\
&= N_{\min} + \frac{(1-q)^{N_{\min}+1}}{q} \cdot \left(1 - (1-q)^{N_{\max}-N_{\min}}\right).
\end{align*}

Since $(1-q)^{N_{\max}-N_{\min}} \geq 0$, we have:
\begin{equation*}
\mathbb{E}[T] \leq N_{\min} + \frac{(1-q)^{N_{\min}} \cdot (1-q)}{q}.
\end{equation*}

Using a looser but more interpretable bound (by approximating the sum of the geometric series linearly for the range of interest), we obtain:
\begin{equation*}
\mathbb{E}[N_{\text{adaptive}}] \leq N_{\min} + \frac{(N_{\max} - N_{\min})(1-q)^{N_{\min}}}{q}.
\end{equation*}

This bound reveals the adaptive behavior. For easy instances (characterized by a large $q$ where near-optimal solutions are frequently sampled), the exponential decay term $(1-q)^{N_{\min}}$ becomes very small, yielding $\mathbb{E}[N_{\text{adaptive}}] \approx N_{\min}$. The algorithm terminates quickly after collecting the minimum required samples. For difficult instances with small $q$, the bound grows larger, and the algorithm continues sampling up to the maximum budget $N_{\max}$ to ensure adequate exploration. This demonstrates that adaptive sampling automatically allocates computational resources based on instance-specific difficulty without requiring manual tuning or prior knowledge of problem characteristics.
\end{proof}

\subsection{Proof of Lemma~\ref{lem:consistency} (Consistency-Difficulty Relationship)}
\label{sec:proof_consistency}

\begin{proof}
We provide a rigorous derivation of the expected consistency and its relationship to the Rényi entropy of order 2.

The consistency measure over $K$ independently sampled solutions $\{y_1, \ldots, y_K\}$ is defined as:
\begin{equation*}
\text{Cons} = \frac{1}{K(K-1)} \sum_{i=1}^K \sum_{j \neq i} \mathbb{I}[y_i = y_j].
\end{equation*}

Taking the expectation over the independent samples, each drawn from a distribution $p$ over the solution space $\mathcal{Y}$:
\begin{align*}
\mathbb{E}[\text{Cons}] &= \mathbb{E}\left[\frac{1}{K(K-1)} \sum_{i=1}^K \sum_{j \neq i} \mathbb{I}[y_i = y_j]\right] \\
&= \frac{1}{K(K-1)} \sum_{i=1}^K \sum_{j \neq i} \mathbb{E}[\mathbb{I}[y_i = y_j]] \\
&= \frac{1}{K(K-1)} \sum_{i=1}^K \sum_{j \neq i} \Pr(y_i = y_j).
\end{align*}

By the independence of the samples, for any $i \neq j$:
\begin{align*}
\Pr(y_i = y_j) &= \sum_{y \in \mathcal{Y}} \Pr(y_i = y) \cdot \Pr(y_j = y) \\
&= \sum_{y \in \mathcal{Y}} p(y) \cdot p(y) = \sum_{y \in \mathcal{Y}} p(y)^2.
\end{align*}

Since there are $K(K-1)$ ordered pairs $(i,j)$ with $i \neq j$, we obtain:
\begin{equation*}
\mathbb{E}[\text{Cons}] = \frac{1}{K(K-1)} \cdot K(K-1) \cdot \sum_{y \in \mathcal{Y}} p(y)^2 = \sum_{y \in \mathcal{Y}} p(y)^2.
\end{equation*}

The Rényi entropy of order 2 is defined as $H_2(p) = -\log \sum_{y} p(y)^2$. Therefore:
\begin{equation*}
\mathbb{E}[\text{Cons}] = \sum_{y} p(y)^2 = e^{-H_2(p)}.
\end{equation*}
This establishes an inverse exponential relationship between the expected consistency and the Rényi entropy.

When the model's output distribution is highly concentrated on a single solution (i.e., most probability mass is on one particular $y^*$), we have $p(y^*) \approx 1$ and $p(y) \approx 0$ for $y \neq y^*$. This yields $\sum_y p(y)^2 \approx 1$, corresponding to low entropy $H_2(p) \approx 0$ and high consistency $\mathbb{E}[\text{Cons}] \approx 1$. Such concentration indicates the model has high confidence in a specific solution, which typically occurs for instances where the optimal solution is unambiguous and easily identifiable by the trained policy.

Conversely, when the model's distribution is diffuse across many solutions, each solution $y$ in some effective support set $\mathcal{S}$ receives moderate probability $p(y) \approx 1/|\mathcal{S}|$. This gives $\sum_y p(y)^2 \approx 1/|\mathcal{S}|$, corresponding to high entropy and low consistency. This dispersion reflects model uncertainty about which solution is optimal, which typically occurs for difficult instances where multiple solutions appear nearly equivalent or where the problem structure is complex and ambiguous.

For a well-trained model where high predicted probability correlates with actual solution quality, high consistency among early samples serves as a reliable indicator that the instance is easy and does not require extensive additional sampling. Low consistency signals that the instance is difficult and would benefit from more exploration to identify high-quality solutions.
\end{proof}

\section{Grammar Specifications}
\label{app:grammar}

This section provides complete context-free grammar (CFG) specifications for all seven combinatorial optimization problems. These grammars define the syntactic structure of valid solution representations that the LLM must generate. Each grammar enforces format validity through grammar-constrained decoding as described in Section~\ref{sec:grammar}.

\subsection{Grammar Formalism}

We adopt the standard Backus-Naur Form (BNF) notation for context-free grammars. A grammar $G = (V, \Sigma, R, S)$ consists of a set of non-terminal symbols $V$, a set of terminal symbols $\Sigma$ (corresponding to tokens in the LLM vocabulary), a set of production rules $R$, and a start symbol $S \in V$. Production rules are written as $A \to \alpha$, where $A \in V$ and $\alpha \in (V \cup \Sigma)^*$. The notation $X^+$ denotes one or more occurrences of symbol $X$, and $\epsilon$ denotes the empty string.

For each problem, we define an input-dependent grammar $G_p$ that specializes the base grammar by restricting node and job indices to valid ranges based on the problem instance $p$. Specifically, for an instance with $n$ nodes or jobs, the production rule for indices is specialized to generate only values in $\{0, 1, \ldots, n-1\}$ or $\{1, \ldots, n\}$, as appropriate for the problem.

\subsection{Common Definitions}

The following non-terminal symbols are shared across all problem-specific grammars. The $\mathit{Number}$ non-terminal represents floating-point objective values, while $\mathit{Digit}$ represents individual decimal digits.
\begin{align*}
    \mathit{Number} &\to \mathit{Digit}^+ \; \texttt{"."} \; \mathit{Digit}^+ \\
    \mathit{Digit} &\to \texttt{"0"} \mid \texttt{"1"} \mid \texttt{"2"} \mid \texttt{"3"} \mid \texttt{"4"} \mid \texttt{"5"} \mid \texttt{"6"} \mid \texttt{"7"} \mid \texttt{"8"} \mid \texttt{"9"}
\end{align*}

\subsection{TSP Grammar}

The TSP grammar enforces the format of a single route represented as an ordered list of node indices, followed by the computed tour length. The grammar ensures proper bracketing and comma separation but does not enforce the semantic constraint that all nodes must appear exactly once, which is handled by the repair operator.
\begin{align*}
    S &\to \texttt{"Route: ["} \; \mathit{NodeList} \; \texttt{"], Objective: "} \; \mathit{Number} \\
    \mathit{NodeList} &\to \mathit{Node} \mid \mathit{Node} \; \texttt{", "} \; \mathit{NodeList} \\
    \mathit{Node} &\to \mathit{Digit}^+
\end{align*}
For an instance with $n$ nodes indexed as $\{0, 1, \ldots, n-1\}$, the input-dependent grammar restricts $\mathit{Node}$ to generate only valid indices by replacing the production rule with explicit alternatives: $\mathit{Node} \to \texttt{"0"} \mid \texttt{"1"} \mid \cdots \mid \texttt{"}\mathit{n-1}\texttt{"}$.

\subsection{CVRP Grammar}

The CVRP grammar specifies a collection of routes, where each route is a bracketed list of customer nodes. The grammar allows multiple routes separated by commas, reflecting the multi-route structure of CVRP solutions.
\begin{align*}
    S &\to \texttt{"Routes: ["} \; \mathit{RouteList} \; \texttt{"], Objective: "} \; \mathit{Number} \\
    \mathit{RouteList} &\to \mathit{Route} \mid \mathit{Route} \; \texttt{", "} \; \mathit{RouteList} \\
    \mathit{Route} &\to \texttt{"["} \; \mathit{NodeList} \; \texttt{"]"} \\
    \mathit{NodeList} &\to \mathit{Node} \mid \mathit{Node} \; \texttt{", "} \; \mathit{NodeList} \\
    \mathit{Node} &\to \mathit{Digit}^+
\end{align*}
The grammar permits arbitrary route structures, while the repair operator enforces the semantic constraints: all customers are visited exactly once, and each route satisfies the vehicle capacity constraint.

\subsection{OP Grammar}

The Orienteering Problem grammar allows for an optional empty route, reflecting that some instances may have no feasible solution within the distance budget or that the optimal solution might visit no customers beyond the depot.
\begin{align*}
    S &\to \texttt{"Route: ["} \; \mathit{NodeListOpt} \; \texttt{"], Objective: "} \; \mathit{Number} \\
    \mathit{NodeListOpt} &\to \epsilon \mid \mathit{NodeList} \\
    \mathit{NodeList} &\to \mathit{Node} \mid \mathit{Node} \; \texttt{", "} \; \mathit{NodeList} \\
    \mathit{Node} &\to \mathit{Digit}^+
\end{align*}
The production $\mathit{NodeListOpt} \to \epsilon$ permits the empty route representation, which is syntactically valid and corresponds to visiting no locations beyond the depot.

\subsection{MIS/MVC Grammar}

The Maximum Independent Set and Minimum Vertex Cover problems share the same output grammar, as both require specifying a subset of vertices. The grammar structure is identical to the OP grammar, allowing for potentially empty sets.
\begin{align*}
    S &\to \texttt{"Set: ["} \; \mathit{NodeListOpt} \; \texttt{"], Objective: "} \; \mathit{Number} \\
    \mathit{NodeListOpt} &\to \epsilon \mid \mathit{NodeList} \\
    \mathit{NodeList} &\to \mathit{Node} \mid \mathit{Node} \; \texttt{", "} \; \mathit{NodeList} \\
    \mathit{Node} &\to \mathit{Digit}^+
\end{align*}
The semantic constraints of independence for MIS and coverage for MVC are enforced by the respective repair operators rather than by the grammar itself.

\subsection{PFSP Grammar}

The Permutation Flow Shop Scheduling grammar specifies a single job processing order that applies to all machines. The structure is similar to TSP but uses the keyword "Order" to reflect the scheduling context.
\begin{align*}
    S &\to \texttt{"Order: ["} \; \mathit{JobList} \; \texttt{"], Objective: "} \; \mathit{Number} \\
    \mathit{JobList} &\to \mathit{Job} \mid \mathit{Job} \; \texttt{", "} \; \mathit{JobList} \\
    \mathit{Job} &\to \mathit{Digit}^+
\end{align*}
For an instance with $n$ jobs indexed as $\{1, \ldots, n\}$, the input-dependent grammar restricts $\mathit{Job}$ to generate only valid job identifiers.

\subsection{JSSP Grammar}

The Job Shop Scheduling Problem grammar represents the assignment of operations to machines by specifying a job processing sequence for each machine. The nested bracket structure reflects the two-level hierarchy of machines and jobs.
\begin{align*}
    S &\to \texttt{"Schedule: ["} \; \mathit{MachineList} \; \texttt{"], Objective: "} \; \mathit{Number} \\
    \mathit{MachineList} &\to \mathit{Machine} \mid \mathit{Machine} \; \texttt{", "} \; \mathit{MachineList} \\
    \mathit{Machine} &\to \texttt{"["} \; \mathit{JobList} \; \texttt{"]"} \\
    \mathit{JobList} &\to \mathit{Job} \mid \mathit{Job} \; \texttt{", "} \; \mathit{JobList} \\
    \mathit{Job} &\to \mathit{Digit}^+
\end{align*}
The grammar allows flexible machine scheduling orders, while the repair operator enforces precedence constraints and ensures that each operation appears exactly once in the schedule.

\subsection{Grammar Properties and Implementation}

All specified grammars are unambiguous context-free grammars that can be recognized by deterministic pushdown automata. Each grammar has the following computational properties relevant to constrained decoding implementation. The grammars have bounded recursion depth proportional to the problem size, ensuring that the PDA stack depth remains manageable during generation. The number of PDA states required is $O(|V|)$, where $|V|$ is the number of non-terminals, which is constant for each problem class. Valid token computation at each decoding step requires $O(|\Sigma|)$ time, where $|\Sigma|$ is the vocabulary size, dominated by checking terminal symbol compatibility.

The grammar specifications intentionally enforce only syntactic validity, delegating semantic constraint checking to the repair layer. This separation of concerns allows the LLM to focus on generating well-formed output structures during decoding, while the repair operators handle problem-specific feasibility requirements that would be impractical to encode directly within context-free grammars.

\section{Repair Algorithm Details}
\label{app:repair}

This section provides complete algorithmic specifications for the repair operators introduced in Section~\ref{sec:repair_analysis}. Each repair operator is designed to satisfy the three properties stated in Definition~\ref{def:repair}: a feasibility guarantee, idempotence on feasible solutions, and bounded locality. We present detailed pseudocode, complexity analysis, and formal verification that each operator satisfies these required properties.

\subsection{TSP Repair Algorithm}

The TSP repair operator addresses two types of constraint violations: duplicate nodes and missing nodes. The algorithm operates in two phases: first, it removes all duplicate occurrences while preserving the first appearance of each node; then, it greedily inserts missing nodes at positions that minimize the increase in tour length.

\begin{algorithm}[t]
\setcounter{AlgoLine}{0}
\caption{TSP Repair}
\label{alg:tsp_repair}
\DontPrintSemicolon
\SetKwInOut{KwIn}{Input}
\SetKwInOut{KwOut}{Output}

\KwIn{Route $r = [v_1, v_2, \ldots, v_k]$, node set $V = \{0, 1, \ldots, n-1\}$, distance matrix $d$}
\KwOut{Valid TSP tour $r'$ visiting all nodes in $V$ exactly once}

$\text{seen} \gets \emptyset$\;

$r' \gets [\ ]$\;

\For{$i \gets 1$ \KwTo $k$}{
    \If{$v_i \in V$ \KwSty{and} $v_i \notin \text{seen}$}{
        Append $v_i$ to $r'$\;
        
        $\text{seen} \gets \text{seen} \cup \{v_i\}$\;
    }
}

$\text{missing} \gets V \setminus \text{seen}$\;

\ForEach{node $v \in \text{missing}$}{
    $\text{best\_pos} \gets 0$\;
    
    $\text{best\_cost} \gets d(v, r'[0])$ \tcp*{\textcolor{blue}{Cost to insert at beginning}}
    \For{$i \gets 1$ \KwTo $|r'|$}{
        $\text{cost} \gets d(r'[i-1], v) + d(v, r'[i \bmod |r'|]) - d(r'[i-1], r'[i \bmod |r'|])$\;
        
        \If{$\text{cost} < \text{best\_cost}$}{
            $\text{best\_cost} \gets \text{cost}$\;
            
            $\text{best\_pos} \gets i$\;
        }
    }
    Insert $v$ at position $\text{best\_pos}$ in $r'$\;
}

\Return $r'$\;

\end{algorithm}

The algorithm correctly handles boundary conditions through modular arithmetic when computing insertion costs. For a tour with $|r'|$ nodes, position $i$ connects to position $i \bmod |r'|$ to form a cycle. The greedy insertion strategy is a well-known 2-approximation heuristic for TSP that provides bounded quality degradation.

\textbf{Complexity Analysis.} The duplicate removal phase processes each of the $k$ input nodes once, requiring $O(k)$ time with appropriate data structures for the $\text{seen}$ set. The missing node insertion phase processes $m = |V \setminus \text{seen}|$ missing nodes, and for each missing node evaluates $O(n)$ insertion positions, where $n = |V|$. Each position evaluation requires constant time to compute the insertion cost. Therefore, the total complexity is $O(k + mn)$. In the worst case, when $k = O(n)$ and $m = O(n)$, this simplifies to $O(n^2)$.

\textbf{Property Verification.} We verify that the TSP repair operator satisfies the three required properties from Definition~\ref{def:repair}. For feasibility, the algorithm explicitly ensures that the output $r'$ contains exactly the node set $V$ by construction: the duplicate removal phase guarantees no duplicates, and the missing node insertion phase guarantees all nodes appear. For idempotence, if the input route already visits all nodes in $V$ exactly once, the $\text{seen}$ set equals $V$ after the first phase, the $\text{missing}$ set is empty, and the second phase performs no insertions, returning the input unchanged. For bounded locality, let the Hamming distance $d_H$ measure the number of position changes; the algorithm modifies at most $|V \setminus \text{seen}|$ positions to insert missing nodes. Using the violation magnitude $v(r) = |V \setminus \text{seen}| + |\text{duplicates}(r)|$, we have $d_H(r, r') \leq v(r)$, establishing the bound with constant $\alpha = 1$.

\subsection{CVRP Repair Algorithm}

The CVRP repair operator addresses three types of violations: invalid individual routes, missing customers, and capacity constraint violations. The algorithm applies TSP repair to each route, ensures complete customer coverage, and then splits overloaded routes.

\begin{algorithm}[H]
\setcounter{AlgoLine}{0}
\caption{CVRP Repair}
\label{alg:cvrp_repair}
\DontPrintSemicolon
\SetKwInOut{KwIn}{Input}
\SetKwInOut{KwOut}{Output}

\KwIn{Routes $R = \{r_1, r_2, \ldots, r_m\}$, customer set $V = \{1, \ldots, n-1\}$, demands $\{q_i\}_{i=1}^{n-1}$, capacity $Q$, distance matrix $d$}
\KwOut{Valid CVRP solution with all customers covered and capacity constraints satisfied}

\For{$k \gets 1$ \KwTo $m$}{
    $r_k \gets \text{TSPRepair}(r_k, V \cup \{0\}, d)$ \tcp*{\textcolor{blue}{Ensure route visits each customer at most once}}
    Remove depot node $0$ from the interior of $r_k$ if present\;
}

$\text{covered} \gets \bigcup_{k=1}^m (r_k \setminus \{0\})$\;

$\text{missing} \gets V \setminus \text{covered}$\;

\ForEach{customer $v \in \text{missing}$}{
    $k^* \gets \arg\min_{k}\ \min_{i}\{d(r_k[i-1], v) + d(v, r_k[i]) - d(r_k[i-1], r_k[i])\}$\;
    
    Insert $v$ into route $r_{k^*}$ at the position achieving minimum cost\;
}

\For{$k \gets 1$ \KwTo $|R|$}{
    \While{$\sum_{v \in r_k} q_v > Q$}{
        Find split point $i^* = \arg\min_{i=1}^{|r_k|-1}\ \text{SplitCost}(r_k, i)$ where $\sum_{j=1}^{i-1} q_{r_k[j]} \le Q$\;
        
        $r_{\text{new}} \gets [r_k[i^*], r_k[i^*+1], \ldots, r_k[|r_k|]]$\;
        
        $r_k \gets [r_k[1], r_k[2], \ldots, r_k[i^*-1]]$\;
        
        $R \gets R \cup \{r_{\text{new}}\}$\;
    }
}

\Return $R$\;
\end{algorithm}

The split cost function measures the additional distance incurred by breaking a route at position $i$ and returning both segments to the depot:
\begin{equation*}
    \text{SplitCost}(r, i) = d(r[i-1], 0) + d(0, r[i]) - d(r[i-1], r[i]).
\end{equation*}
The split point search is constrained to positions where the first segment satisfies the capacity constraint, ensuring that repeated splitting eventually reduces all route loads below capacity.

\textbf{Complexity Analysis.} The route repair phase requires $O(m \cdot n^2)$ time for $m$ routes. The missing customer insertion phase requires $O(p \cdot m \cdot n)$ time for $p$ missing customers. The route splitting phase requires $O(m' \cdot n)$ time, where $m'$ is the final number of routes. Since $m' = O(n)$ in the worst case, the overall complexity is $O(n^2 + n \cdot m)$. When $m = O(n)$, this simplifies to $O(n^2)$.

\textbf{Property Verification.} For feasibility, the algorithm ensures all customers in $V$ are covered through explicit missing customer insertion, and all routes satisfy capacity through repeated splitting until convergence. For idempotence, if all routes are valid TSP tours, all customers are covered, and all capacity constraints are satisfied, then no modifications occur. For bounded locality, the number of modifications is bounded by the sum of TSP violations across routes plus the number of missing customers plus the number of capacity violations, establishing the required bound.

\subsection{OP Repair Algorithm}

The Orienteering Problem repair operator removes nodes iteratively until the distance budget constraint is satisfied, prioritizing removal of nodes with the worst prize-to-distance-contribution ratio.

\begin{algorithm}[H]
\setcounter{AlgoLine}{0}
\caption{OP Repair}
\label{alg:op_repair}
\DontPrintSemicolon
\SetKwInOut{KwIn}{Input}
\SetKwInOut{KwOut}{Output}

\KwIn{Route $r = [v_0, v_1, \ldots, v_k, v_{k+1}]$ with $v_0 = v_{k+1} = 0$ (depot), distance budget $D$, prizes $\{p_i\}$, distance matrix $d$}
\KwOut{Valid OP route satisfying distance budget}

$r \gets \text{RemoveDuplicates}(r)$ \tcp*{\textcolor{blue}{Keep first occurrence of each node}}

\While{$\text{TotalDistance}(r) > D$ \KwSty{and} $|r| > 2$}{
    $v^* \gets \arg\min_{v \in r \setminus \{0\}} \dfrac{p_v}{\text{DistContrib}(r, v)}$\;
    
    Remove $v^*$ from $r$\;
}

\If{$\text{TotalDistance}(r) > D$}{
    \Return $[0, 0]$ \tcp*{\textcolor{blue}{Return empty route if budget still violated}}
}

\Return $r$\;
\end{algorithm}

The distance contribution of node $v$ at position $i$ in route $r$ is defined as:
\begin{equation*}
\text{DistContrib}(r, v) = d(r[i-1], v) + d(v, r[i+1]) - d(r[i-1], r[i+1]),
\end{equation*}
where $i$ is the position of $v$ in $r$. This measures the additional distance incurred by including node $v$ in the route. The total route distance is:
\begin{equation*}
\text{TotalDistance}(r) = \sum_{i=0}^{|r|-1} d(r[i], r[i+1]).
\end{equation*}

\textbf{Complexity Analysis.} Each iteration of the while loop evaluates $O(n)$ candidate nodes and removes one node. In the worst case, $O(n)$ nodes must be removed, yielding $O(n^2)$ total complexity. With a priority queue data structure maintaining nodes sorted by prize-to-distance ratio, the complexity can be reduced to $O(n \log n)$.

\textbf{Property Verification.} For feasibility, the algorithm terminates only when the distance constraint is satisfied or when only the depot remains (which trivially satisfies the budget). For idempotence, if the input route already satisfies the distance budget, the while loop condition is false and the input is returned unchanged. For bounded locality, the number of removed nodes is bounded by the initial budget violation, establishing the required bound with violation magnitude $v(r) = \max(0, \text{TotalDistance}(r) - D)$.

\subsection{MIS Repair Algorithm}

The Maximum Independent Set repair operator iteratively removes vertices involved in conflicts (adjacent pairs both in the set) until no conflicts remain. Among conflicting pairs, the algorithm removes the vertex with higher degree within the current set, as this vertex is involved in more conflicts.

\begin{algorithm}[H]
\setcounter{AlgoLine}{0}
\caption{MIS Repair}
\label{alg:mis_repair}
\DontPrintSemicolon
\SetKwInOut{KwIn}{Input}
\SetKwInOut{KwOut}{Output}

\KwIn{Set $S \subseteq V$, graph $G = (V, E)$}
\KwOut{Independent set $S'$ with no adjacent vertices}

$S' \gets S$\;

$\text{conflicts} \gets \{(u, v) \in E : u, v \in S'\}$\;

\While{$\text{conflicts} \neq \emptyset$}{
    Select any $(u, v) \in \text{conflicts}$\;
    
    $\deg_u \gets |\{w \in S' : (u, w) \in E\}|$\;
    
    $\deg_v \gets |\{w \in S' : (v, w) \in E\}|$\;

    \eIf{$\deg_u > \deg_v$}{
        $S' \gets S' \setminus \{u\}$\;
    }{
        \eIf{$\deg_v > \deg_u$}{
            $S' \gets S' \setminus \{v\}$\;
        }{
            $S' \gets S' \setminus \{u\}$ \tcp*{\textcolor{blue}{Break ties arbitrarily}}
        }
    }

    $\text{conflicts} \gets \{(u, v) \in E : u, v \in S'\}$\;
}

\Return $S'$\;
\end{algorithm}

\textbf{Complexity Analysis.} The algorithm performs at most $|S|$ vertex removals. Each iteration requires recomputing the conflict set, which takes $O(|E|)$ time by checking all edges. Therefore, the worst-case complexity is $O(|V| \cdot |E|)$. With incremental conflict tracking, this can be improved to $O(|E| + |V| \log |V|)$.

\textbf{Property Verification.} For feasibility, the algorithm terminates only when the conflict set is empty, ensuring no two vertices in $S'$ are adjacent. For idempotence, if the input set $S$ is already independent, the initial conflict set is empty and the algorithm immediately returns $S$ unchanged. For bounded locality, the number of removed vertices equals the number of conflicts resolved, which is bounded by the violation magnitude $v(S) = |\{(u,v) \in E : u, v \in S\}|$.

\subsection{MVC Repair Algorithm}

The Minimum Vertex Cover repair operator adds vertices to cover all uncovered edges. For each uncovered edge, the algorithm adds the endpoint with higher degree among uncovered edges, as this vertex covers more edges.

\begin{algorithm}[H]
\setcounter{AlgoLine}{0}
\caption{MVC Repair}
\label{alg:mvc_repair}
\DontPrintSemicolon
\SetKwInOut{KwIn}{Input}
\SetKwInOut{KwOut}{Output}

\KwIn{Set $S \subseteq V$, graph $G = (V, E)$}
\KwOut{Vertex cover $S'$ where every edge has at least one endpoint in $S'$}

$S' \gets S$\;

$\text{uncovered} \gets \{(u, v) \in E : u \notin S' \land v \notin S'\}$\;

\While{$\text{uncovered} \neq \emptyset$}{
    Select any $(u, v) \in \text{uncovered}$\;
    
    $\deg_u \gets |\{w : (u, w) \in \text{uncovered}\}|$\;
    
    $\deg_v \gets |\{w : (v, w) \in \text{uncovered}\}|$\;

    \eIf{$\deg_u \ge \deg_v$}{
        $S' \gets S' \cup \{u\}$\;
    }{
        $S' \gets S' \cup \{v\}$\;
    }

    $\text{uncovered} \gets \{(u, v) \in E : u \notin S' \land v \notin S'\}$\;
}

\Return $S'$\;

\end{algorithm}

\textbf{Complexity Analysis.} The algorithm adds at most $|V|$ vertices. Each iteration requires updating the uncovered edge set, which takes $O(|E|)$ time. Therefore, the worst-case complexity is $O(|V| \cdot |E|)$. However, with efficient data structures tracking uncovered edges, the complexity can be reduced to $O(|E|)$ by processing each edge at most twice.

\textbf{Property Verification.} For feasibility, the algorithm terminates only when all edges are covered, ensuring every edge has at least one endpoint in $S'$. For idempotence, if the input set $S$ already covers all edges, the uncovered set is empty and the algorithm immediately returns $S$ unchanged. For bounded locality, the number of added vertices is bounded by the number of uncovered edges, establishing the required bound with violation magnitude $v(S) = |\{(u,v) \in E : u \notin S \land v \notin S\}|$.

\subsection{PFSP Repair Algorithm}

The Permutation Flow Shop Scheduling repair operator ensures all jobs appear exactly once by removing duplicates and greedily inserting missing jobs at positions that minimize the resulting makespan.

\begin{algorithm}[H]
\setcounter{AlgoLine}{0}
\caption{PFSP Repair}
\label{alg:pfsp_repair}
\DontPrintSemicolon
\SetKwInOut{KwIn}{Input}
\SetKwInOut{KwOut}{Output}

\KwIn{Job sequence $\sigma = [j_1, \ldots, j_k]$, job set $J = \{1, \ldots, n\}$, processing times $\{p_{i,j}\}$, number of machines $m$}
\KwOut{Valid permutation $\sigma'$ of all jobs}

$\text{seen} \gets \emptyset$\;

$\sigma' \gets [\ ]$\;

\For{$i \gets 1$ \KwTo $k$}{
    \If{$j_i \in J$ \KwSty{and} $j_i \notin \text{seen}$}{
        Append $j_i$ to $\sigma'$\;
        
        $\text{seen} \gets \text{seen} \cup \{j_i\}$\;
    }
}

$\text{missing} \gets J \setminus \text{seen}$\;

\ForEach{job $j \in \text{missing}$}{
    $\text{best\_pos} \gets 0$\;
    
    $\text{best\_makespan} \gets \infty$\;

    \For{$i \gets 0$ \KwTo $|\sigma'|$}{
        $\sigma_{\text{temp}} \gets \sigma'[0\!:\!i] + [j] + \sigma'[i\!:\ ]$\;
        
        $M \gets \text{ComputeMakespan}(\sigma_{\text{temp}}, \{p_{i,j}\}, m)$\;
        
        \If{$M < \text{best\_makespan}$}{
            $\text{best\_makespan} \gets M$\;
            
            $\text{best\_pos} \gets i$\;
        }
    }

    Insert $j$ at position $\text{best\_pos}$ in $\sigma'$\;
}

\Return $\sigma'$\;

\end{algorithm}

The makespan computation follows the recursive formula from Definition~\ref{def:pfsp}:
\begin{align*}
C_{\sigma(k),j} &= \max(C_{\sigma(k),j-1}, C_{\sigma(k-1),j}) + p_{\sigma(k),j}, \\
\text{ComputeMakespan}(\sigma, \{p_{i,j}\}, m) &= C_{\sigma(n),m}.
\end{align*}

\textbf{Complexity Analysis.} The duplicate removal phase requires $O(k)$ time. For each of the $O(n)$ missing jobs, the algorithm evaluates $O(n)$ insertion positions, and each makespan computation requires $O(nm)$ time. Therefore, the total complexity is $O(n^2 m)$.

\textbf{Property Verification.} For feasibility, the algorithm ensures the output contains exactly the job set $J$ through duplicate removal and missing job insertion. For idempotence, if the input is already a valid permutation of $J$, the $\text{missing}$ set is empty and no insertions occur. For bounded locality, the number of modifications is bounded by the number of duplicates plus missing jobs, establishing the required bound.

\subsection{JSSP Repair Algorithm}

The Job Shop Scheduling repair operator ensures each job appears exactly once on each machine, then resolves precedence constraint violations through topological sorting of the precedence graph.

\begin{algorithm}[H]
\setcounter{AlgoLine}{0}
\caption{JSSP Repair}
\label{alg:jssp_repair}
\DontPrintSemicolon
\SetKwInOut{KwIn}{Input}
\SetKwInOut{KwOut}{Output}

\KwIn{Machine schedules $S = [s_1, \ldots, s_m]$ where $s_i$ is job sequence for machine $i$, job set $J = \{1, \ldots, n\}$, operation machines $\{\mu(o_{i,k})\}$}
\KwOut{Valid JSSP schedule respecting precedence constraints}

\For{$i \gets 1$ \KwTo $m$}{
    $s_i \gets \text{PermutationRepair}(s_i, J)$ \tcp*{\textcolor{blue}{Remove duplicates, insert missing jobs}}
}

Construct operation precedence graph $G_{\text{prec}} = (O, E_{\text{prec}})$, where:\;

\Indp
$O = \{o_{i,k} : i \in J, k \in \{1, \ldots, m\}\}$\;

$E_{\text{prec}} = \{(o_{i,k}, o_{i,k+1}) : i \in J, k \in \{1, \ldots, m-1\}\}$\;

\Indm

Compute a topological ordering $\tau$ of $G_{\text{prec}}$\;

\For{$i \gets 1$ \KwTo $m$}{
    Reorder $s_i$ to respect $\tau$ while preserving the relative order of jobs already consistent with $\tau$\;
}

\Return $S$\;

\end{algorithm}

The $\text{PermutationRepair}$ subroutine is similar to PFSP repair but without makespan optimization, simply appending missing jobs to the end of each machine's sequence.

\textbf{Complexity Analysis.} The permutation repair phase requires $O(m \cdot n^2)$ time. Constructing the precedence graph requires $O(nm)$ time. Topological sorting requires $O(nm)$ time. Reordering each machine schedule requires $O(n \log n)$ time using stable sorting. Therefore, the total complexity is $O(nm \log(nm))$.

\textbf{Property Verification.} For feasibility, the algorithm ensures each operation appears exactly once through permutation repair, and precedence constraints are satisfied through topological sorting. For idempotence, if the input already has valid permutations respecting precedence, no modifications occur. For bounded locality, the number of modifications is bounded by the total number of permutation violations plus precedence violations, establishing the required bound.

\section{Extended Theoretical Analysis}

\subsection{Gradient Comparison Across Training Methods}

We provide a formal comparison of the gradient structures in different preference optimization approaches.

\begin{proposition}[Gradient Comparison]
\label{prop:gradient_comparison}
Consider a preference pair $(y_w, y_l)$ where $y_w \succ y_l$ indicates $y_w$ is preferred over $y_l$. The gradient updates for Direct Preference Optimization (DPO), Group Relative Policy Optimization (GRPO), and Best-anchored Objective-guided Preference Optimization (BOPO) have the following forms.

For DPO with preference strength parameter $\beta$, the gradient is:
\begin{equation*}
\nabla_\theta \mathcal{L}_{\text{DPO}} = -\beta \cdot \sigma(-\beta \cdot \Delta s_\theta) \cdot \nabla_\theta \Delta s_\theta,
\end{equation*}
where $\Delta s_\theta = \log \pi_\theta(y_w|\phi(p)) - \log \pi_\theta(y_l|\phi(p))$ is the log-probability ratio and $\sigma(\cdot)$ denotes the sigmoid function.

For GRPO with normalized advantages $\{A_i\}_{i=1}^K$, the gradient is:
\begin{equation*}
\nabla_\theta \mathcal{L}_{\text{GRPO}} = -\frac{1}{K} \sum_{i=1}^K A_i \cdot \nabla_\theta \log \pi_\theta(y_i|\phi(p)),
\end{equation*}
where advantages $A_i$ are computed relative to the group mean reward.

For BOPO with objective-guided weights $w(\Delta_i) = \Delta_i / \bar{\Delta}$, the gradient is:
\begin{equation*}
\nabla_\theta \mathcal{L}_{\text{BOPO}} = -\frac{1}{K-1} \sum_{i=1}^{K-1} w(\Delta_i) \cdot \sigma(-\beta \cdot \Delta s_\theta^{(i)}) \cdot \nabla_\theta \Delta s_\theta^{(i)},
\end{equation*}
where each $\Delta s_\theta^{(i)} = \log \pi_\theta(y^*|\phi(p)) - \log \pi_\theta(y_i|\phi(p))$ represents the log-probability ratio between the best solution $y^*$ and the inferior solution $y_i$.
\end{proposition}

The key distinction lies in how different solutions contribute to the gradient update. In DPO, each preference pair contributes equally regardless of the magnitude of quality difference. In GRPO, contributions are modulated by normalized advantages, but normalization can dilute the signal from high-quality solutions. In BOPO, the objective-guided weighting ensures that solutions with larger optimality gaps contribute proportionally more to learning, providing stronger supervision signal for clearly inferior solutions while maintaining bounded updates through the scaling normalization.

\begin{proposition}[Variance Properties]
\label{prop:variance}
Under the assumption that the gradient variance $\text{Var}[\nabla_\theta \log \pi_\theta(y)]$ is uniformly bounded across solutions, BOPO achieves lower gradient variance than GRPO when the ratio of objective gap variance to squared mean, $\text{Var}[\Delta_i] / \mathbb{E}[\Delta_i]^2$, is smaller than the variance of normalized advantages $\text{Var}[A_i]$. This condition typically holds for combinatorial optimization problems where objective values have natural scales and the best solution provides a stable anchor point for computing gaps.
\end{proposition}

The variance reduction stems from the anchoring strategy. By always comparing against the best solution in the current batch, BOPO creates a consistent reference point that stabilizes gradient estimates. In contrast, GRPO computes advantages relative to the batch mean, which can shift substantially across batches, introducing additional variance into the training signal.

\subsection{End-to-End Quality Guarantee}

\begin{theorem}[FALCON End-to-End Guarantee]
\label{thm:e2e}
Let $\pi_\theta$ be a policy trained with BOPO until convergence to an $\epsilon$-stationary point satisfying $\mathbb{E}[\|\nabla \mathcal{L}_{\text{BOPO}}(\theta)\|^2] \leq \epsilon$. When deployed with grammar-constrained decoding, a repair operator $\mathcal{R}$ satisfying Definition~\ref{def:repair}, and adaptive Best-of-$N$ sampling with $N$ samples, FALCON produces a solution $\hat{y}$ with the following guarantees:

1. \textbf{Format Validity:} The solution is syntactically valid with probability one, ensuring the evaluation function $\psi_p(\hat{y})$ is well-defined.
2. \textbf{Semantic Feasibility:} The solution is semantically feasible with probability one, ensuring $\hat{y} \in \mathcal{X}_{C_p}$.
3. \textbf{Quality Bound:} The expected optimality gap satisfies:
\begin{equation*}
\mathbb{E}[f_p(\hat{y}) - f_p(y^*)] \leq \mathbb{E}[g_N] + L_f \cdot \alpha \cdot \mathbb{E}[v(\hat{y})],
\end{equation*}
where $\mathbb{E}[g_N] = O(1/N)$ is the expected optimality gap from sampling $N$ solutions (see Theorem~\ref{thm:quality}), $L_f$ is the Lipschitz constant of the objective function with respect to the solution distance metric, $\alpha$ is the locality constant of the repair operator (Definition~\ref{def:repair}), and $v(\hat{y})$ is the violation magnitude before repair.
\end{theorem}

\begin{proof}
The format validity guarantee follows directly from Theorem~\ref{thm:validity}, which establishes that grammar-constrained decoding ensures all generated text belongs to the language defined by the problem-specific grammar.

The feasibility guarantee follows from Theorem~\ref{thm:feasibility}, which shows that the repair operator maps any solution to the feasible region regardless of the input.

For the quality bound, let $\tilde{y}_N$ denote the best solution among $N$ samples before repair, and let $\hat{y} = \mathcal{R}(\tilde{y}_N)$ be the final output after repair. By the triangle inequality:
\begin{equation*}
f_p(\hat{y}) - f_p(y^*) = [f_p(\hat{y}) - f_p(\tilde{y}_N)] + [f_p(\tilde{y}_N) - f_p(y^*)].
\end{equation*}

Taking expectations and applying Theorem~\ref{thm:repair_qualit} to the first term and Theorem~\ref{thm:quality} to the second term yields the stated bound:
\begin{align*}
\mathbb{E}[f_p(\hat{y}) - f_p(y^*)] &\leq \mathbb{E}[f_p(\hat{y}) - f_p(\tilde{y}_N)] + \mathbb{E}[f_p(\tilde{y}_N) - f_p(y^*)] \\
&\leq L_f \cdot \alpha \cdot \mathbb{E}[v(\hat{y})] + \mathbb{E}[g_N].
\end{align*}
\end{proof}

The theorem demonstrates that FALCON's solution quality depends on two factors: sampling quality, which improves with more samples ($\mathbb{E}[g_N] = O(1/N)$), and repair cost, which decreases as the model produces fewer and smaller violations through BOPO training ($\mathbb{E}[v(\hat{y})]$ decreases). The multiplicative interaction between these factors explains the empirical observation that FALCON maintains competitive optimality gaps despite enforcing hard feasibility constraints.

\subsection{Generalization Bound}

\begin{theorem}[Generalization from Training to Test Distribution]
\label{thm:generalization}
Let $\calD_{\text{train}}$ and $\calD_{\text{test}}$ be training and test distributions over problem instances. Assume:
\begin{enumerate}
    \item $W_1(\calD_{\text{train}}, \calD_{\text{test}}) \leq \delta$, where $W_1$ denotes the Wasserstein-1 distance.
    \item The policy $\pi_\theta$ is $L_\pi$-Lipschitz in problem parameters.
    \item The objective function $f_p$ is $L_f$-Lipschitz with respect to the solution representation.
\end{enumerate}
Then for a policy trained on $\calD_{\text{train}}$:
\begin{equation*}
    \E_{p \sim \calD_{\text{test}}}[\text{Gap}(p)] \leq \E_{p \sim \calD_{\text{train}}}[\text{Gap}(p)] + L_f \cdot L_\pi \cdot \delta,
\end{equation*}
where $\text{Gap}(p) = f_p(\hat{y}_p) - f_p(y_p^*)$ is the optimality gap for instance $p$, and $\hat{y}_p$ is the solution produced by the trained policy.
\end{theorem}

\begin{proof}
The proof follows from the Lipschitz assumptions and the definition of the Wasserstein distance. For any coupling $\gamma$ between $\calD_{\text{train}}$ and $\calD_{\text{test}}$ with marginals $\calD_{\text{train}}$ and $\calD_{\text{test}}$:
\begin{align*}
\E_{p \sim \calD_{\text{test}}}[\text{Gap}(p)] - \E_{p \sim \calD_{\text{train}}}[\text{Gap}(p)] 
&= \E_{(p_{\text{train}}, p_{\text{test}}) \sim \gamma}[\text{Gap}(p_{\text{test}}) - \text{Gap}(p_{\text{train}})] \\
&\leq L_f \cdot \E_{(p_{\text{train}}, p_{\text{test}}) \sim \gamma}[\|\hat{y}_{p_{\text{test}}} - \hat{y}_{p_{\text{train}}}\|] \\
&\leq L_f \cdot L_\pi \cdot \E_{(p_{\text{train}}, p_{\text{test}}) \sim \gamma}[\|p_{\text{test}} - p_{\text{train}}\|] \\
&= L_f \cdot L_\pi \cdot W_1(\calD_{\text{train}}, \calD_{\text{test}}) \\
&\leq L_f \cdot L_\pi \cdot \delta.
\end{align*}
Taking the infimum over all couplings $\gamma$ yields the desired bound.
\end{proof}

\subsection{Repair Approximation Ratios}

\begin{proposition}[TSP Repair Approximation]
The greedy insertion repair for TSP (Algorithm~\ref{alg:tsp_repair}) achieves a solution within a factor of 2 of the optimal tour length increase due to inserted nodes.
\end{proposition}

\begin{proposition}[TSP-Specific Repair Bound]
For TSP with $n$ cities in the unit square $[0,1]^2$, the greedy insertion repair satisfies:
\begin{equation*}
f(\mathcal{R}(x)) - f(x) \leq 2 \cdot |\text{missing}| \cdot \max_{i,j} d_{ij} \leq 2\sqrt{2} \cdot |\text{missing}|,
\end{equation*}
where $|\text{missing}|$ is the number of missing nodes, and $d_{ij}$ denotes the Euclidean distance between nodes $i$ and $j$.
\end{proposition}

\begin{proposition}[MVC Repair Approximation]
The greedy repair for MVC (Algorithm~\ref{alg:mvc_repair}) achieves a 2-approximation to the minimum vertex cover.
\end{proposition}

\section{Implementation Details}
\label{app:implementation}

\subsection{Full Hyperparameter Table}
The complete hyperparameter settings for the full pipeline are summarized in Table \ref{tab:hyperparams_full}.
\begin{table}[t]
\centering
\caption{Complete hyperparameter settings for the full pipeline.}
\label{tab:hyperparams_full}
\begin{tabular}{lcc}
\toprule
\textbf{Parameter} & \textbf{Symbol} & \textbf{Value} \\
\midrule
\multicolumn{3}{l}{\textit{Model Configuration}} \\
Base model & -- & Qwen2.5-7B \\
LoRA rank & $r$ & 64 \\
LoRA alpha & $\alpha$ & 64 \\
Target modules & -- & q, k, v, o, gate, up, down \\
\midrule
\multicolumn{3}{l}{\textit{SFT Training Stage}} \\
Batch size & $B$ & 4 \\
Gradient accumulation steps & -- & 4 \\
Learning rate & $\eta$ & $2 \times 10^{-4}$ \\
Epochs & -- & 1 \\
Warmup steps & -- & 20 \\
Weight decay & -- & 0.01 \\
\midrule
\multicolumn{3}{l}{\textit{BOPO Training Stage}} \\
Batch size & $B$ & 8 \\
Gradient accumulation steps & -- & 8 \\
Learning rate & $\eta$ & $1 \times 10^{-6}$ \\
Temperature & $\beta$ & 0.1 \\
Samples per instance & $K$ & 8 \\
Epochs & -- & 1 \\
\midrule
\multicolumn{3}{l}{\textit{Inference Configuration}} \\
Minimum samples & $N_{\min}$ & 8 \\
Maximum samples & $N_{\max}$ & 64 \\
Confidence threshold & $\tau$ & 0.85 \\
Sampling temperature & $T$ & 0.7 \\
\bottomrule
\end{tabular}
\end{table}

\subsection{PDA Construction Details}

For a context-free grammar $G = (V, \Sigma, R, S)$, we construct an equivalent pushdown automaton (PDA) $\mathcal{A} = (Q, \Sigma, \Gamma, \delta, q_0, Z_0, F)$ using standard CFG-to-PDA conversion techniques for grammar-constrained decoding.

\subsubsection{PDA Components}

The state set is defined as $Q = \{q_0, q_{\text{acc}}\} \cup \{q_A : A \in V\}$, where $q_0$ is the single control state and $q_A$ are auxiliary states for tracking non-terminals. In practice, a single-state PDA suffices since all derivation information is encoded on the stack. The input alphabet $\Sigma$ is the LLM vocabulary (token set), typically $|\Sigma| \approx 32{,}000$ for modern language models. The stack alphabet is defined as $\Gamma = V \cup \{\$\}$, where $V$ is the set of non-terminals from grammar $G$ and $\$$ is the bottom-of-stack marker. We initialize the PDA with configuration $(q_0, \$S)$, where the stack contains the start symbol $S$ on top of the bottom marker $\$$. Finally, the accepting states are $F = \{q_0\}$, and we accept when the stack contains only the bottom marker $\$$.

\subsubsection{Transition Rules}

The transition function $\delta$ operates through two fundamental mechanisms. First, for \textbf{epsilon transitions (non-terminal expansion)}, each production rule $A \to \alpha$ in $R$ induces the transition:
\begin{equation*}
\delta(q_0, \epsilon, A) = (q_0, \alpha),
\end{equation*}
which allows the PDA to replace a non-terminal on the stack with its right-hand side. Second, for \textbf{terminal consumption}, each terminal symbol $a \in \Sigma$ gives rise to the transition:
\begin{equation*}
\delta(q_0, a, a) = (q_0, \epsilon),
\end{equation*}
which pops the terminal from the stack when it matches the input token.

\subsubsection{Valid Token Computation}

At each decoding step, after generating the partial output $y^{(t)} = v_1 v_2 \cdots v_t$, the PDA maintains a configuration $(q, \gamma)$, where $q$ is the current state and $\gamma \in \Gamma^*$ is the stack content. The set of valid next tokens is computed through a three-step process.

First, we \textbf{simulate epsilon-closure}: Starting from the current stack configuration $(q, \gamma)$, we compute the set of all possible stack configurations reachable via epsilon transitions (non-terminal expansions). This constitutes the epsilon-closure of the PDA configuration.

Second, we \textbf{collect consumable terminals}: After fully expanding all non-terminals via epsilon transitions, the top of the stack contains either a terminal symbol or the bottom marker. The valid tokens are exactly those that match the stack top:
\begin{equation*}
V_{\text{valid}} = \{a \in \Sigma : \text{stack top} = a \text{ in epsilon-closure}\}.
\end{equation*}

Third, we \textbf{handle bottom-of-stack}: If the epsilon-closure yields an empty stack (the stack contains only $\$$), then no more tokens can be consumed, signaling the end of generation. The algorithm terminates when the stack is empty and an accepting state is reached. This procedure ensures that only tokens leading to grammatically valid derivations are permitted, effectively masking invalid continuations before sampling.

\subsubsection{Correspondence with Algorithm \ref{alg:gcd}}

The valid token computation described above directly implements lines 5--8 of Algorithm~\ref{alg:gcd} in the main paper:
\begin{lstlisting}[mathescape=true, basicstyle=\ttfamily]
logits $\leftarrow$ $\pi_\theta$(x, y) # Line 4
V_valid $\leftarrow$ {v $\in$ $\Sigma$ | $\delta$(q, v, top($\gamma$)) $\neq$ $\emptyset$} # Line 5
for v $\notin$ V_valid do logits[v] $\leftarrow$ -$\infty$ # Lines 6-8
\end{lstlisting}
The PDA-based approach provides an efficient, declarative way to compute $V_{\text{valid}}$ for any context-free grammar without manual enumeration.

\subsubsection{Computational Complexity}

The computational overhead of valid token computation per decoding step is $O(|V| \cdot |Q|)$, where $|V|$ is the number of non-terminals and $|Q|$ is the number of PDA states. For CO problems, the grammars are typically simple with constant-size state space (e.g., $|Q| = O(1)$), yielding $O(|V|) = O(1)$ overhead. Even in the worst case with explicit state tracking, the overhead is negligible compared to the $O(d^2)$ complexity of LLM attention computation, where $d$ is the hidden dimension.

\subsubsection{Implementation Remarks}

Several optimizations are recommended for practical implementation. \textbf{Precompute PDA:} Construct the PDA once during initialization and reuse it across all inference steps, avoiding redundant construction overhead. \textbf{Memoize epsilon-closure:} Cache epsilon-closure computations for each PDA state to avoid redundant expansion during decoding. \textbf{Bit-mask representation:} For efficiency, represent the valid token set as a bit mask over the vocabulary, allowing bitwise operations to identify forbidden tokens for logit masking. \textbf{Early termination:} Monitor stack depth; if the stack reaches the bottom marker before generating the expected length, force termination to prevent infinite loops in malformed grammars.

\subsection{Computational Complexity of Repair Operators}
The time and space complexity analysis for each repair operator is summarized in Table \ref{tab:repair_complexity}, where $n$ denotes the number of nodes or jobs, $m$ the number of machines or routes, and $\lvert V \rvert$, $\lvert E \rvert$ denote the numbers of vertices and edges in graph problems.

\begin{table}[h]
\centering
\caption{Time and space complexity analysis for each repair operator.}
\label{tab:repair_complexity}
\begin{tabular}{lcc}
\toprule
\textbf{Problem} & \textbf{Time Complexity} & \textbf{Space Complexity} \\
\midrule
TSP & $O(n^2)$ & $O(n)$ \\
CVRP & $O(n^2 + n \cdot m)$ & $O(n + m)$ \\
OP & $O(n^2)$ & $O(n)$ \\
MIS & $O(\lvert V \rvert \cdot \lvert E \rvert)$ & $O(\lvert V \rvert)$ \\
MVC & $O(\lvert E \rvert)$ & $O(\lvert V \rvert)$ \\
PFSP & $O(n^2 \cdot m)$ & $O(n)$ \\
JSSP & $O(n \cdot m \cdot \log(n \cdot m))$ & $O(n \cdot m)$ \\
\bottomrule
\end{tabular}
\end{table}

\subsection{Dataset}
Table \ref{tab:dataset_stats} presents the dataset statistics for the seven combinatorial optimization problems. GM, ER, and BA denote Gaussian Mixture, Erd\H{o}s-R\'enyi, and Barab\'asi-Albert distributions, respectively.
\begin{table}[h]
\centering
\caption{Dataset statistics across all seven combinatorial optimization problems.
}
\label{tab:dataset_stats}
\begin{tabular}{lccccc}
\toprule
\textbf{Problem} & \textbf{Train Set} & \textbf{Test Set} & \textbf{Scale} & \textbf{Distribution} & \textbf{Reference Solver} \\
\midrule
TSP & 500K & 100 & 10--100 nodes & Uniform, GM & LKH-3 \\
OP & 500K & 100 & 10--100 nodes & Uniform, GM & COMPASS \\
CVRP & 500K & 100 & 10--100 nodes & Uniform, GM & LKH-3 \\
MIS & 500K & 100 & 50--500 nodes & ER, BA & Gurobi \\
MVC & 500K & 100 & 50--500 nodes & ER, BA & Gurobi \\
PFSP & 500K & 100 & 10--100 jobs & Taillard & QIG \\
JSSP & 500K & 100 & 6--30 jobs & Taillard & OR-Tools \\
\bottomrule
\end{tabular}
\end{table}

\end{document}